\documentclass[notitlepage]{article}

\usepackage{arxiv}

\usepackage[ruled,vlined,linesnumbered]{algorithm2e}
\usepackage{amsthm}
\usepackage{amsfonts}       
\usepackage{amssymb}
\usepackage{booktabs}       
\usepackage{bm}
\usepackage{enumitem}
\usepackage[T1]{fontenc}    
\usepackage{hyperref}       
\usepackage[utf8]{inputenc} 
\usepackage{url}            
\usepackage{makecell}
\usepackage{mathtools}
\usepackage{multirow}
\usepackage{nicefrac}       
\usepackage{microtype}      
\usepackage[textsize=tiny]{todonotes}
\usepackage{thmtools}
\usepackage{xcolor}         

\hypersetup{
  colorlinks=true,
  linkcolor=blue,
  citecolor=blue
}

\usepackage[
  backend=biber,
  style=numeric-comp,
  maxbibnames=9,
  maxcitenames=1,
  uniquelist=false,
  uniquename=false,
  isbn=false,
  sorting=nty,
  doi=false
]{biblatex}
\SetKwInput{KwInput}{Input}
\SetKwInput{KwInit}{Initialize}
\SetKwInput{KwTest}{Test}
\SetKwInput{KwDef}{Definition}
\SetKwInput{KwReturn}{Return}

\makeatletter
\patchcmd{\@algocf@start}
  {-1.5em}
  {0pt}
  {}{}
\makeatother

\newtheorem{theorem}{Theorem}
\newtheorem{lemma}[theorem]{Lemma}

\newtheorem{corollary[theorem]}{Corollary}

\newtheorem{remark}{Remark}
\newtheorem{assumption}{Assumption}

\newtheorem*{assumption*}{Assumption}

\let\oldremark\remark
\renewcommand{\remark}{\oldremark\normalfont}

\DeclareMathOperator*{\argmax}{argmax}

\newcommand{\red}[1]{\textcolor{red}{#1}}
\definecolor{kh}{HTML}{FFA500}
\definecolor{yz}{HTML}{6495ED}
\definecolor{yellow}{HTML}{F3E5AB}

\newcount\Comments  
\newcommand{\kibitz}[2]{\ifnum\Comments=1\textcolor{#1}{#2}\fi}

\title{Reinforcement Learning for Infinite-Horizon Average-Reward Linear MDPs via Approximation by Discounted-Reward MDPs}

\author{%
  Kihyuk Hong \\
  University of Michigan\\
  \texttt{kihyukh@umich.edu} \\
  \And
  Woojin Chae \\
  KAIST \\
  \texttt{woojeeny02@kaist.ac.kr} \\
  \And
  Yufan Zhang \\
  University of Michigan \\
  \texttt{yufanzh@umich.edu} \\
  \AND
  Dabeen Lee \\
  KAIST\\
  \texttt{dabeenl@kaist.ac.kr} \\
  \And
  Ambuj Tewari \\
  University of Michigan \\
  \texttt{tewaria@umich.edu} \\
}

\begin{document}

\maketitle

\begin{abstract}
We study the problem of infinite-horizon average-reward reinforcement learning with linear Markov decision processes (MDPs).
The associated Bellman operator of the problem not being a contraction makes the algorithm design challenging.
Previous approaches either suffer from computational inefficiency or require strong assumptions on dynamics, such as ergodicity, for achieving a regret bound of $\widetilde{\mathcal{O}}(\sqrt{T})$.
In this paper, we propose the first algorithm that achieves $\widetilde{\mathcal{O}}(\sqrt{T})$ regret with computational complexity polynomial in the problem parameters, without making strong assumptions on dynamics.
Our approach approximates the average-reward setting by a discounted MDP with a carefully chosen discounting factor, and then applies an optimistic value iteration.
We propose an algorithmic structure that plans for a nonstationary policy through optimistic value iteration and follows that policy until a specified information metric in the collected data doubles. 
Additionally, we introduce a value function clipping procedure for limiting the span of the value function for sample efficiency.
\end{abstract}

\section{Introduction}

Reinforcement learning (RL) in the infinite-horizon average-reward setting aims to learn a policy that maximizes the average reward in the long run.
This setting is relevant for applications where the interaction between the agent and environment does not terminate and continues indefinitely, such as in network routing \parencite{mammeri2019reinforcement} or inventory management \parencite{giannoccaro2002inventory}.
Designing algorithms in this setting is challenging because the associated Bellman operator is not a contraction.
This complicates the use of optimistic value iteration-based algorithms, which add bonus terms to the value function every iteration, a widely used approach in the finite-horizon episodic~\parencite{jin2020provably} and infinite-horizon discounted settings~\parencite{he2021nearly}.

The seminal work of \parencite{jaksch2010near} adopts a model-based approach in the tabular setting that constructs a confidence set on the transition model.
Their algorithm runs optimistic value iterations by choosing an optimistic model from the confidence set at each iteration.
Most work in the tabular setting follows this approach of constructing confidence sets for the model \parencite{bartlett2009regal,fruit2018efficient}.
Other work~\parencite{wei2020model} that does not use this approach either has suboptimal regret bound or assumes the ergodicity assumption.

Adapting the approach of constructing confidence set on the transition model to function approximation settings, such as linear MDP setting, with arbitrarily large state space is challenging because sample-efficient model estimation is elusive in general unless additional assumptions on the transition model are made.
Previous work on infinite-horizon average-reward RL with function approximation either imposes ergodicity assumptions~\parencite{wei2021learning} or uses computationally inefficient algorithms \parencite{wei2021learning,he2024sample} with time complexity exponential in problem parameters to achieve $\widetilde{\mathcal{O}}(\sqrt{T})$ regret.
The following question remains open:
\begin{quote}
\textit{Does there exist a polynomial-time algorithm for infinite-horizon average-reward linear MDPs that achieves $\widetilde{\mathcal{O}}(\sqrt{T})$ regret without requiring the ergodicity assumption?}
\end{quote}

In this paper, we answer the question in the affirmative.
We draw insights from a recent line of work~\parencite{wei2020model,zhang2023sharper} that employs the technique of approximating infinite-horizon average-reward RL by a discounted MDP in the tabular setting.
We apply the discounted setting approximation idea to the linear MDP setting and design an optimistic value iteration algorithm that achieves $\widetilde{\mathcal{O}}(\text{sp}(v^\ast) \sqrt{d^3 T})$ regret without making the ergodicity assumption.
A key component of our method is the use of value function clipping, coupled with a novel value iteration scheme to ensure efficient learning in this setting.

The rest of the paper is organized as follows.
In Section~\ref{section:preliminaries}, we formally define infinite-horizon average-reward setting and discounted setting.
In Section~\ref{section:tabular}, we introduce a simple value iteration based algorithm for the tabular setting that approximates the average-reward setting by the discounted setting.
In Section~\ref{section:linear}, we adapt the algorithm design to the linear MDP setting, and addresses a unique challenge in this setting by proposing a novel algorithm structure.

\begin{table*}[t]
\caption{Comparison of algorithms for infinite-horizon average-reward linear MDP}
\label{table:comparison}
\centering
\begin{tabular}{cccc}
 \toprule
 Algorithm & Regret $\widetilde{\mathcal{O}}(\cdot)$ & Assumption & Computation $\text{poly}(\cdot)$ \\
 \midrule
 FOPO \parencite{wei2021learning} & $\text{sp}(v^\ast) \sqrt{d^3 T}$ & Bellman optimality equation & \red{$T^d, A, d$} \\
 OLSVI.FH \parencite{wei2021learning} & \red{$\sqrt{\text{sp}(v^\ast) }(d T)^{\frac{3}{4}}$} & Bellman optimality equation & $T, A, d$ \\
 LOOP \parencite{he2024sample} & $\sqrt{\text{sp}(v^\ast)^3 d^3 T}$ & Bellman optimality equation & \red{$T^d, A, d$} \\
 MDP-EXP2 \parencite{wei2021learning} & $d \sqrt{t_{\text{mix}}^3 T}$ & \red{Uniform mixing} & $T, A, d$ \\
 \textbf{$\gamma$-LSCVI-UCB (Ours)} & $\text{sp}(v^\ast) \sqrt{d^3 T}$ & Bellman optimality equation & $T, S, A, d$ \\
 \midrule
 Lower Bound \parencite{wu2022nearly} & $\Omega(d \sqrt{\text{sp}(v^\ast) T})$ & & \\
 \bottomrule
\end{tabular}
\end{table*}

\subsection{Related Work}
A comparison of our work to previous work on infinite-horizon average-reward linear MDPs is shown in Table~\ref{table:comparison}.
Entries highlighted in red indicate suboptimality compared to our algorithm.
Our algorithm is the first to achieve $\widetilde{\mathcal{O}}(\sqrt{T})$ regret with computation complexity polynomial in the parameters $d, S, A, T$ without making the ergodicity assumption.
Note that the uniform mixing assumption required for MDP-EXP2~\parencite{wei2021learning} is a stronger assumption than the ergodicity assumption.
Our regret matches that of FOPO~\parencite{wei2021learning}, an algorithm that requires solving a computationally intractable optimization problem for finding optimistic value function.
A brute-force approach for solving their optimization problem requires computations polynomial in $T^d$, which is exponential in $d$.
In contrast, our algorithm's complexity is polynomial in $d$.
However, it depends polynomially on the size of the state space $S$, whereas the complexity of FOPO does not depend on $S$.
This implies our algorithm is an improvement over FOPO in the regime where $S \ll T^d$.
We leave the problem of getting rid of the dependence on $S$ for future work.
Additional comparisons in the tabular setting, as well as a broader discussion of related work, can be found in Appendix~\ref{section:additional-related-work}.

\paragraph{Approximation of Average-Reward Setting by Discounted Setting}
The technique of approximating the average-reward setting to the discounted setting is used by
\textcite{jin2021towards,wang2022near,zurek2023span,wang2023optimal} to solve the sample complexity problem of producing a nearly optimal policy given access to a simulator in the tabular setting.
\textcite{wei2020model} use the reduction for the online RL setting with tabular MDPs. They propose a Q-learning based algorithm, but has $\widetilde{\mathcal{O}}(T^{2/3})$ regret.
\textcite{zhang2023sharper} also use the reduction for the online RL setting with tabular MDPs. Their algorithm is also Q-learning based, but they improve the regret to $\widetilde{\mathcal{O}}(\sqrt{T})$ by introducing a novel method for estimating the span.

\paragraph{Infinite-Horizon Average-Reward Setting with Linear Mixture MDPs}
The linear mixture MDP setting is closely related to the linear MDP setting in that the Bellman operator admits a compact representation.
However, linear mixture MDP parameterizes the probability transition model such that the transition probability is linear in the low-dimensional feature representation of state-action-state triplets.
Such a structure allows a sample efficient estimation of the model, enabling the design of an optimism-based algorithm using a confidence set on the model, much like the model-based approach for the tabular setting.
\textcite{wu2022nearly} and \textcite{ayoub2020model} design an optimism-based algorithm using a confidence set under the assumption that MDP is communicating.
\textcite{chae2024learning} use the method of reducing the average-reward setting to the discounted setting, and achieve a nearly minimax optimal regret bound under a weaker assumption on the MDP.

\section{Preliminaries} \label{section:preliminaries}

\paragraph{Notations}
Let $\Vert \bm{x} \Vert_A = \sqrt{x^T A x}$ for $\bm{x} \in \mathbb{R}^d$ and a psd matrix $A \in \mathbb{R}^{d \times d}$.
Let $a \vee b = \max \{ a, b \}$ and $a \wedge b = \min\{a, b \}$.
Let $\Delta(\mathcal{X})$ be the set of probability measures on $\mathcal{X}$.
Let $[n] = \{1, \dots, n\}$ and $[m:n] = \{m, m  + 1, \dots, n\}$.
Let $\text{sp}(v) = \max_{s, s'} \vert v(s) - v(s') \vert$.
Let $[Pv](s, a) = \mathbb{E}_{s' \sim P(\cdot | s, a)}[v(s')]$.

\subsection{Infinite-Horizon Average-Reward MDPs}

We consider a Markov decision process (MDP)~\parencite{puterman2014markov}, $\mathcal{M} = (\mathcal{S}, \mathcal{A}, P, r)$ where $\mathcal{S}$ is the state space, $\mathcal{A}$ is the action space, $P : \mathcal{S} \times \mathcal{A} \rightarrow \Delta(\mathcal{S})$ is the probability transition kernel, $r : \mathcal{S} \times \mathcal{A} \rightarrow [0, 1]$ is the reward function.
We assume $\mathcal{S}$ is a measurable space with possibly infinite number of elements and $\mathcal{A}$ is a finite set.
We assume the reward is deterministic and the reward function $r$ is known to the learner.
The probability transition kernel $P$ is unknown to the learner. 

The interaction protocol between the learner and the MDP is as follows.
The learner interacts with the MDP for $T$ steps, starting from an arbitrary state $s_1 \in \mathcal{S}$ chosen by the environment.
At each step $t = 1, \dots, T$, the learner chooses an action $a_t \in \mathcal{A}$ and observes the reward $r(s_t, a_t)$ and the next state $s_{t + 1}$.
The next state $s_{t + 1}$ is drawn by the environment from $P(\cdot | s_t, a_t)$.

Consider a stationary policy $\pi : \mathcal{S} \rightarrow \Delta(\mathcal{A})$ where $\pi(a | s)$ specifies the probability of choosing action $a$ at state $s$.
The performance measure of our interest for the policy $\pi$ is the long-term average reward starting from an initial state $s$ defined as
$$
J^\pi(s) \coloneqq \liminf_{T \rightarrow \infty} \frac{1}{T} \mathbb{E}^\pi \left[
\sum_{t = 1}^T r(s_t, a_t) | s_1 = s
\right]
$$
where $\mathbb{E}^\pi[ \cdot ]$ is the expectation with respect to the probability distribution on the trajectory $(s_1, a_1, s_2, a_2, \dots)$ induced by the interaction between $P$ and $\pi$.
The performance of the learner is measured by the regret against the best stationary policy $\pi^\ast$ that maximizes $J^\pi(s_1)$.
Writing $J^\ast(s_1) \coloneqq J^{\pi^\ast}(s_1)$, the regret is$$
R_T \coloneqq \sum_{t = 1}^T (J^\ast(s_1) - r(s_t, a_t)).
$$
As discussed by \textcite{bartlett2009regal}, without an additional assumption on the structure of the MDP, if the agent enters a bad state from which reaching the optimally rewarding states is impossible, the agent may suffer a linear regret.
To avoid this pathological case, we follow \textcite{wei2021learning} and make the following structural assumption on the MDP.
\begin{assumption}[Bellman optimality equation] \label{assumption:bellman-optimality}
There exist $J^\ast \in \mathbb{R}$ and functions $v^\ast : \mathcal{S} \rightarrow \mathbb{R}$ and $q^\ast : \mathcal{S} \times \mathcal{A} \rightarrow \mathbb{R}$ such that for all $(s, a) \in \mathcal{S} \times \mathcal{A}$, we have
\begin{align*}
J^\ast + q^\ast(s, a) &= r(s, a) + [Pv^\ast](s, a) \\
v^\ast(s) &= \max_{a \in \mathcal{A}} q^\ast(s, a).
\end{align*}
\end{assumption}
As discussed by \textcite{wei2021learning}, the Bellman optimality equation assumption is a weaker assumption than the weakly communicating assumption, which in turn is weaker than ergodicity.
Weakly communicating assumption is another widely used assumption for the infinite-horizon average-reward setting that requires each pair of states in the set to be reachable from each other under some policy.
Ergodicity assumption requires the Markov chain induced by any policy to be ergodic, i.e., irreducible and aperiodic.
Uniform mixing assumption required for MDP-EXP2 proposed by \textcite{wei2021learning} is a stronger assumption than ergodicity that additionally assumes that the mixing time is uniformly bounded over all policies.

As shown by \textcite{wei2021learning}, under Assumption~\ref{assumption:bellman-optimality}, the policy $\pi^\ast$ that deterministically selects an action from $\argmax_a q^\ast(s, a)$ at each state $s \in \mathcal{S}$ is an optimal policy.
Moreover, $\pi^\ast$ always gives an optimal average reward $J^{\pi^\ast}(s_1) = J^\ast$ for all initial states $s_1 \in \mathcal{S}$.
Since the optimal average reward is independent of the initial state, we can simply write the regret as $R_T = \sum_{t = 1}^T (J^\ast - r(s_t, a_t))$.
Functions $v^\ast(s)$ and $q^\ast(s, a)$ are the relative advantage of starting with $s$ and $(s, a)$ respectively.
We can expect a problem with large $\text{sp}(v^\ast)$ to be more difficult since starting with a bad state can be more disadvantageous.
As is common in the literature \parencite{bartlett2009regal,wei2020model}, we assume $\text{sp}(v^\ast)$ is known to the learner.

\begin{remark}
Instead of assuming exact knowledge of $\text{sp}(v^\ast)$, we can assume that the learner knows an upper bound $H$.
In this case, using $H$ as an input to our algorithm instead of $\text{sp}(v^\ast)$, the regret bound of our algorithm will scale with $H$ rather than $\text{sp}(v^\ast)$.
Such a knowledge of an upper bound is a commonly made assumption~\parencite{bartlett2009regal}.
In practice, if one has a general sense of the diameter of the MDP, which is the expected number of steps needed to transition between any two states in the worst case, the diameter can serve as an upper bound $H$, since the diameter is guaranteed to be an upper bound of $\text{sp}(v^\ast)$ when reward is bounded by 1.
Relaxing the assumption of the knowledge of $\text{sp}(v^\ast)$ or its upper bound is only achieved recently in the tabular setting~\parencite{boone2024achieving}.
Extending this relaxation to the linear MDP setting remains an open question for future work.
\end{remark}

\subsection{Infinite-Horizon Discounted Setting}

The key idea of this paper, inspired by \textcite{zhang2023sharper}, is to approximate the infinite-horizon average-reward setting by the infinite-horizon discounted setting with a discount factor $\gamma \in [0, 1)$ tuned carefully.
Introducing the discounting factor allows for a computationally efficient algorithm design that exploits the contraction property of the Bellman operator for the infinite-horizon discounted setting.
When $\gamma$ is close to 1, we expect the optimal policy for the discounted setting to be nearly optimal for the average-reward setting, given the classical result \parencite{puterman2014markov} that says the average reward of a stationary policy is equal to the limit of the discounted cumulative reward as $\gamma$ goes to $1$.
Before stating a lemma that relates the infinite-horizon average-reward setting and the infinite-horizon discounted setting, we define the value function under the discounted setting.
For a policy $\pi$, define
\begin{align*}
V^\pi(s) &= \mathbb{E}^\pi\left[\sum_{t = 1}^\infty \gamma^{t - 1} r(s_t, a_t) | s_1 = s\right] \\
Q^\pi(s, a) &= \mathbb{E}^\pi \left[ \sum_{t = 1}^\infty \gamma^{t - 1} r(s_t, a_t) | s_1 = s, a_1 = a \right].
\end{align*}
We suppress the dependency of the value functions on the discounting factor $\gamma$.
We write the optimal value functions under the discounted setting as
$$
V^\ast(s) = \max_\pi V^\pi(s), \quad Q^\ast(s, a) = \max_\pi Q^\pi(s, a).
$$
The following lemma relates the infinite-horizon average-reward setting and the discounted setting.

\begin{lemma}[Lemma 2 in \textcite{wei2020model}] \label{lemma:discounted-approximation}
For any $\gamma \in [0, 1)$, the optimal value function $V^\ast$ for the infinite-horizon discounted setting with discounting factor $\gamma$ satisfies
\begin{enumerate}[label=(\roman*)]
\item $\text{sp}(V^\ast) \leq 2 \text{sp} (v^\ast)$ and
\item $\vert (1 - \gamma) V^\ast(s) - J^\ast \vert \leq ( 1 - \gamma) \text{sp}(v^\ast) ~~\text{for all}~ s \in \mathcal{S}$.
\end{enumerate}
\end{lemma}
The lemma above suggests that the difference between the optimal average reward $J^\ast$ and the optimal discounted cumulative reward normalized by the factor $(1 - \gamma)$ is small as long as $\gamma$ is close to 1.
Hence, we can expect the policy optimal under the discounted setting will be nearly optimal for the average-reward setting, provided $\gamma$ is sufficiently close to 1.

\section{Warmup: Tabular Setting} \label{section:tabular}

In this section, we introduce an algorithm designed for the tabular setting, where the state space $\mathcal{S}$ and action space $\mathcal{A}$ are both finite, and no specific structure is assumed for the reward function or the transition probabilities.
The structure of the algorithm, along with the accompanying analysis, will lay the groundwork for extending these results to the linear MDP setting.
\begin{algorithm*}[t]
\KwInput{Discounting factor $\gamma \in [0, 1)$, span $H$, bonus factor $\beta$.}
\KwInit{$Q_1(s, a), V_1(s) \leftarrow \frac{1}{1 - \gamma}$; $N_0(s,a,s')\leftarrow 0, N_0(s,a)\leftarrow 1$, for all $(s,a,s')\in \mathcal{S}\times\mathcal{A}\times\mathcal{S}$.}
Receive initial state $s_1$. \\
\For{time step $t = 1, \dots, T$}{
    Take action $a_t = \argmax_a Q_t(s_t, a)$. Receive reward $r(s_t, a_t)$. Receive next state $s_{t + 1}$. \\
    $N_t(s_t, a_t, s_{t+1}) \leftarrow N_{t-1}(s_t,a_t,s_{t+1})+1$ \\
    $N_t(s_t, a_t) \leftarrow N_{t-1}(s_t,a_t)+1$. \\
    (Other entries of $N_t$ remain the same as $N_{t-1}$.)\\
    $\widehat{P}_t(s' | s, a) \leftarrow N_t(s, a, s')/ N_t(s, a), ~~\forall (s, a) \in \mathcal{S} \times \mathcal{A}$. \\
    $Q_{t + 1}(s, a) \leftarrow (r(s, a) + \gamma [\widehat{P}_t V_t](s, a) + \beta / \sqrt{N_t(s, a)}) \wedge Q_t(s, a), ~~\forall (s, a) \in \mathcal{S} \times \mathcal{A}$ \label{alg-line:q} \\
    $\widetilde{V}_{t + 1}(s) \leftarrow (\max_a Q_{t + 1}(s, a)) \wedge V_t(s), ~~\forall s \in \mathcal{S}$. \label{alg-line:v-tilde} \\
    $V_{t + 1}(s) \leftarrow \widetilde{V}_{t + 1}(s) \wedge (\min_{s'} \widetilde{V}_{t + 1}(s') + H), ~~\forall s \in \mathcal{S}$. \label{alg-line:projection}\\
}
\caption{$\gamma$-UCB-CVI for Tabular Setting}
\label{alg:ucbvi}
\end{algorithm*}

\subsection{Algorithm}

Our algorithm, called \textit{discounted upper confidence bound clipped value iteration} ($\gamma$-UCB-CVI), adapts UCBVI  \parencite{azar2017minimax}, which was originally designed for the finite-horizon episodic setting, to the infinite-horizon discounted setting.
At each time step, the algorithm performs an approximate Bellman backup with an added bonus term $\beta \sqrt{1 / N_t(s, a)}$ (Line~\ref{alg-line:q}) where $N_t(s, a)$ is the number of times the state-action pair $(s, a)$ is visited.
The bonus term is designed to guarantee optimism, ensuring that $Q_t \geq Q^\ast$ for all $t = 1, \dots, T$.
A key modification from UCBVI is the clipping step (Line~\ref{alg-line:projection}), which bounds span of the value function estimate $V_t$ by $H$, where the target span $H$ is an input to the algorithm.
Without clipping, the span of the value function $V_t$ can be as large as $\frac{1}{1 - \gamma}$, while with clipping, the span can only be as large as $H$.
As we will see in the analysis, this clipping step is crucial to achieving a sharp dependence on $\frac{1}{1 - \gamma}$ in the regret bound, which enables the $\widetilde{O}(\sqrt{T})$ regret through tuning $\gamma$.
Running the algorithm with the discounting factor set to $\gamma = 1 - 1 / \sqrt{T}$ and the target span set to $H = 2 \cdot \text{sp}(v^\ast)$ guarantees the following regret bound.

\begin{theorem} \label{thm:tabular}
Under Assumption~\ref{assumption:bellman-optimality}, there exists a constant $c>0$ such that, for any fixed $\delta\in(0,1)$, if  \mbox{Algorithm \ref{alg:ucbvi}} is run with $\gamma=1-\sqrt{1/T}$, $H=2 \cdot sp(v^\ast)$, and $\beta = cH \sqrt{S \log(SAT/\delta)}$, then with probability at least $1-\delta$, the total regret is bounded by
$$
R_T \leq \mathcal{O}\left(\text{sp}(v^\ast) \sqrt{S^2AT\log(SAT/\delta)}\right).
$$
\end{theorem}

In the theorem above, the constant $c$ in the definition of $\beta$ is determined in Lemma~\ref{lemma:model-error-tabular} in the next subsection.
Our regret bound matches the best previously known regret bound for computationally efficient algorithm in this setting.
See Appendix~\ref{section:additional-related-work} for a full comparison with previous work on infinite-horizon average-reward tabular MDPs.
We believe we can improve our bound by a factor of $\sqrt{S}$ with a refined analysis using Bernstein inequality, following the idea of the refined analysis for UCBVI provided by \textcite{azar2017minimax}.
We leave the improvement to future work.

\subsection{Analysis}

In this section, we outline the proof of Theorem~\ref{thm:tabular}.
We defer the complete proof to Appendix~\ref{section:analysis-tabular}.
The key to the proof is the following concentration inequality.

\begin{lemma} \label{lemma:model-error-tabular}
Under the setting of Theorem \ref{thm:tabular}, there exists a constant $c$ such that for any fixed $\delta\in(0,1)$, we have with probability at least $1 - \delta$ that
$$
\vert [(\widehat{P}_t - P) V_t](s, a) \vert \leq c\cdot\text{sp}(v^\ast) \sqrt{S \log(SAT/\delta)/ N_t(s, a)}
$$
for all $(s, a, t) \in \mathcal{S} \times \mathcal{A} \times [T]$.
\end{lemma}
Without clipping, the span of $V_t$ would be $\frac{1}{1 - \gamma}$ instead of $2 \cdot \text{sp}(v^\ast)$, making the bound of $[(\widehat{P}_t - P)V_t](s, a)$ scale with $\frac{1}{1 - \gamma}$ instead of $\text{sp}(v^\ast)$.
Replacing the $\frac{1}{1 - \gamma}$ factor by $\text{sp}(v^\ast)$ by clipping is crucial to achieving $\widetilde{\mathcal{O}}(\sqrt{T})$ regret when tuning $\gamma$.

In Theorem~\ref{thm:tabular}, the bonus factor parameter $\beta$ is chosen according to the concentration bound in the lemma above, ensuring that the concentration bound for $[(\widehat{P}_t - P)V_t](s, a)$ is $\beta / \sqrt{N_t(s, a)}$, which is the bonus term used by the algorithm.
With this result, we can now establish the following optimism result.

\begin{lemma}[Optimism] \label{lemma:optimism-tabular}
Under the setting of Theorem~\ref{thm:tabular}, we have with probability at least $1 - \delta$ that
$$
V_t(s) \geq V^\ast(s), \quad Q_t(s, a) \geq Q^\ast(s, a)
$$
for all $(s,a,t)\in \mathcal{S}\times\mathcal{A}\times[T]$.
\end{lemma}

The proof uses a standard induction argument (e.g. Lemma 18 in \textcite{azar2017minimax}) to show $\widetilde{V}_t(s) \geq V^\ast(s)$.
To establish that the clipped value function $V_t$, no larger than $\widetilde{V}_t$ by design, still satisfies $V_t(s) \geq V^\ast(s)$, we use $\text{sp}(V^\ast) \leq 2 \cdot \text{sp}(v^\ast)$ (Lemma~\ref{lemma:discounted-approximation}), which guarantees the clipping operation does not reduce $V_t$ below $V^\ast$.

Now, we show the regret bound under the high probability events in the previous two lemmas (Lemma~\ref{lemma:model-error-tabular}, Lemma~\ref{lemma:optimism-tabular}) hold.
By the value iteration step (Line~\ref{alg-line:q}) of Algorithm~\ref{alg:ucbvi} and the concentration inequality in Lemma~\ref{lemma:model-error-tabular},
we have for all $t = 2, \dots, T$ that
\begin{align*}
r(s_t, a_t)
&\geq Q_t(s_t, a_t) - \gamma [\widehat{P}_{t - 1} V_{t - 1}](s_t, a_t) - \beta 
 / \sqrt{N_{t - 1}(s_t, a_t)} \\
&\geq V_t(s_t) - \gamma [PV_{t - 1}](s_t, a_t) - 2 \beta / \sqrt{N_{t - 1}(s_t, a_t)}
\end{align*}
where the second inequality follows by $V_t(s_t) \leq \widetilde{V}_t(s_t) \leq \max_a Q_t(s_t, a) = Q_t(s_t, a_t)$.
Hence, the regret can be bounded by
\begin{align*}
R_T
&=
\sum_{t = 1}^T (J^\ast - r(s_t, a_t)) \\
&\leq \sum_{t = 2}^T (J^\ast - V_t(s_t) + \gamma [PV_{t - 1}](s_t, a_t) + 2 \beta / \sqrt{N_{t - 1}(s_t, a_t)}) + \mathcal{O}(1),
\end{align*}
where the first inequality uses the fact that $J^\ast \leq 1$,
which can be decomposed into
\begin{align*}
&=
\underbrace{ \sum_{t = 2}^T (J^\ast - (1 - \gamma) V_t(s_t))}_{(a)} + \gamma \underbrace{\sum_{t = 2}^T (V_{t - 1}(s_{t + 1}) - V_t(s_t))}_{(b)} \\
&\hspace{10mm}+ \gamma \underbrace{\sum_{t = 2}^T (P V_{t - 1}(s_t, a_t) - V_{t - 1}(s_{t + 1}))}_{(c)} + 2\underbrace{ \beta \sum_{t = 2}^T 1 / \sqrt{N_{t - 1}(s_t, a_t)}}_{(d)}~+~\mathcal{O}(1).
\end{align*}
We bound the terms $(a), (b), (c), (d)$ separately.
\paragraph{Bounding Term $\bm{(a)}$}
Using the optimism result (Lemma~\ref{lemma:optimism-tabular}) that says $V_t(s) \geq V^\ast(s)$ for all $s \in \mathcal{S}$ and Lemma~\ref{lemma:discounted-approximation} that bounds $\vert J^\ast - (1 - \gamma) V^\ast(s) \vert$ for all $s \in \mathcal{S}$, we get
\begin{align*}
\sum_{t = 2}^T (J^\ast - (1 - \gamma)V_t(s_t))
&\leq \sum_{t = 2}^T (J^\ast - (1 - \gamma) V^\ast(s_t)) \\
&\leq T(1 - \gamma) \text{sp}(v^\ast).
\end{align*}
\paragraph{Bounding Term $\bm{(b)}$}
Note that for any $s \in \mathcal{S}$, the sequence $\{ V_t(s) \}_{t = 1}^T$ is monotonically decreasing due to Line~\ref{alg-line:v-tilde}-\ref{alg-line:projection} in Algorithm~\ref{alg:ucbvi}.
Moreover, since $V_t(s) \in [0, \frac{1}{1 - \gamma}]$ for all $t = 1, \dots, T$, the total decrease in $V_t(s)$ from $t = 1$ to $T$ is bounded above by $\frac{1}{1 - \gamma}$. Hence,
\begin{align*}
\sum_{t = 2}^T (V_{t - 1}(s_{t + 1}) - V_t(s_t))
&\le
\sum_{t = 2}^T (V_{t - 1}(s_{t + 1}) - V_{t + 1}(s_{t + 1})) + \mathcal{O}\left( \frac{1}{1 - \gamma} \right) \\
&\le
\sum_{s \in \mathcal{S}} \sum_{t = 2}^T (V_{t - 1}(s) - V_{t + 1}(s)) + \mathcal{O} \left( \frac{1}{1 - \gamma} \right) \\
&\le
\mathcal{O}\left(\frac{S}{1 - \gamma} \right).
\end{align*}

The bound above is polynomial in $S$, the size of the state space, which is undesirable in the linear MDP setting where $S$ can be arbitrarily large.
The main challenge of this paper, as we will see in the next section, is sidestepping this issue for linear MDPs.

\paragraph{Bounding Term $\bm{(c)}$}

Term $(c)$ is the sum of a martingale difference sequence where each term is bounded by $\text{sp}(v^\ast)$. Hence, by the Azuma-Hoeffding inequality, with probability at least $1-\delta$, term $(c)$ is bounded by $\text{sp}(v^\ast)\sqrt{2T\log(1/\delta)}$.
Without clipping, each term of the martingale difference sequence can only be bounded by $\frac{1}{1 - \gamma}$, leading to a bound of $\widetilde{\mathcal{O}}(\frac{1}{1 - \gamma} \sqrt{T})$, which is too loose for achieving a regret bound of $\widetilde{\mathcal{O}}(\sqrt{T})$.

\paragraph{Bounding Term $\bm{(d)}$}
We can bound the sum of the bonus terms $(d)$ using a standard argument (\textcite{azar2017minimax}, Lemma~\ref{lemma:sum-of-bonus-terms-tabular} in Appendix~\ref{section:analysis-tabular}) by $\mathcal{O}(\beta \sqrt{SAT}) = \mathcal{O}(\text{sp}(v^\ast) \sqrt{S^2 AT \log (SAT / \delta)})$.

Combining the above, and rescaling $\delta$, it follows that with probability at least $1-\delta$, we have
\begin{align*}
R_T &\leq \mathcal{O}\Big(T(1 - \gamma) \text{sp}(v^\ast) + \frac{S}{1 - \gamma}
+ \text{sp}(v^\ast) \sqrt{T \log(1 / \delta)} + \text{sp}(v^\ast)\sqrt{S^2AT\log(SAT/\delta)}\Big).
\end{align*}
Choosing $\gamma = 1 - 1 / \sqrt{T}$, we get
$$
R_T \leq \mathcal{O}\left(\text{sp}(v^\ast) \sqrt{S^2AT\log(SAT/\delta)}\right),
$$
which completes the proof of Theorem~\ref{thm:tabular}.

\section{Linear MDP Setting} \label{section:linear}

In this section, we apply the key ideas developed from the previous section to the linear MDP setting, which we formally define below.
The tabular setting studied in the previous section has regret bound that scales polynomially with the size of the state space $S$.
This is because, in the tabular setting, there is no structure in the state space that can be exploited to generalize to unseen states during learning.
Therefore, when learning in a large state space with $S \gg T$, additional structural assumption are necessary.
The linear MDP setting~\parencite{jin2020provably} is a widely used setting in the RL theory literature that allows for generalization to unseen states by introducing structure in the MDP through a low-dimensional state-action feature mapping.
The additional assumption made in the linear MDP setting is as follows.

\begin{assumption}[Linear MDP \parencite{jin2020provably}] \label{assumption:linear-mdp}
We assume that the transition and the reward functions can be expressed as a linear function of a known $d$-dimensional feature map $\bm\varphi : \mathcal{S} \times \mathcal{A} \rightarrow \mathbb{R}^d$ such that for any $(s, a) \in \mathcal{S} \times \mathcal{A}$, we have
$$
r(s, a) = \langle \bm\varphi(s, a), \bm\theta \rangle, \quad
P(s' | s, a) = \langle \bm\varphi(s, a), \bm\mu(s') \rangle
$$
where $\bm\mu(\cdot) = (\mu_1(\cdot), \dots, \mu_d(\cdot))$ is a vector of $d$ unknown measures on $\mathcal{S}$ and $\bm\theta \in \mathbb{R}^d$ is a known parameter for the reward function.
\end{assumption}

As is commonly done in the literature on linear MDPs~\parencite{jin2020provably}, we further assume, without loss of generality (see \textcite{wei2021learning} for justification), the following boundedness conditions:
\begin{equation}\label{eqn:boundedness}
\begin{aligned}
\Vert \bm\varphi(s, a) \Vert_2 \leq 1 ~~ \text{for all}~(s, a) \in \mathcal{S}\times \mathcal{A}, \\
\quad \Vert \bm\theta \Vert_2 \leq \sqrt{d}, \quad \Vert \bm\mu(\mathcal{S}) \Vert_2 \leq \sqrt{d}.
\end{aligned}
\end{equation}

As discussed by \textcite{jin2020provably}, although the transition model $P$ is linear in the $d$-dimensional feature mapping $\bm\varphi$, $P$ still has infinite degrees of freedom as the measure $\bm\mu$ is unknown, making the estimation of the model $P$ difficult.
For sample efficient learning, we leverage the fact that $[Pv](s, a)$ is linear in $\bm\varphi(s, a)$ for any function $v : \mathcal{S} \rightarrow  \mathbb{R}$ so that $[Pv](s, a) = \langle \bm\varphi(s, a), \bm{w}^\ast_v \rangle$ for some $\bm{w}^\ast_v$, since
\begin{align*}
[Pv](s, a) &\coloneqq \int_{s' \in \mathcal{S}} v(s') P(ds' | s, a) \\
&= \int_{s' \in \mathcal{S}} v(s') \langle \bm\varphi(s, a), \bm\mu(ds') \rangle \\
&= \langle \bm\varphi(s, a), \int_{s' \in \mathcal{S}} v(s') \bm\mu(ds') \rangle.
\end{align*}

\begin{algorithm*}[t] \label{alg:lscvi-ucb}
\KwInput{Discounting factor $\gamma \in (0, 1)$, regularization $\lambda > 0$, span $H$, bonus factor $\beta$.}
\KwInit{$t \leftarrow 1$, $k \leftarrow 1$, $t_k \leftarrow 1$, $\Lambda_1 \leftarrow \lambda I$, $\bar\Lambda_0 \leftarrow \lambda I$, $Q^1_t(\cdot, \cdot) \leftarrow \frac{1}{1 - \gamma}$ for $t \in [T]$.}
Receive state $s_1$. \\
\For{time step $t = 1, \dots, T$}{
    Take action $a_t = \argmax_a Q^k_t(s_t, a)$. Receive reward $r(s_t, a_t)$. Receive next state $s_{t + 1}$. \\
    $\bar\Lambda_t \leftarrow \bar\Lambda_{t - 1} + \bm\varphi(s_t, a_t) \bm\varphi(s_t, a_t)^T$. \\
    \vspace{1mm}
    \If{$2 \det (\Lambda_k) < \det(\bar\Lambda_t)$}{ \label{alg-line:episode}
        $k \leftarrow k + 1$, $t_k \leftarrow t + 1$, $\Lambda_k \leftarrow \bar\Lambda_t$. \\
        \tcp{Run value iteration to plan for remaining $T - t_k + 1$ time steps in the new episode.}
        $\widetilde{V}_{T + 1}^k(\cdot) \leftarrow \frac{1}{1 - \gamma}$, $V^k_{T + 1}(\cdot) \leftarrow \frac{1}{1 - \gamma}$. \label{alg-line:value-iteration-start}\\
        \For{$u = T, T - 1, \dots, t_k$}{
            $\bm{w}^k_{u + 1} \leftarrow \Lambda_k^{-1} \sum_{\tau = 1}^{t_k - 1} \bm\varphi(s_\tau, a_\tau) (V^k_{u + 1}(s_{\tau + 1}) - \min_{s'} \widetilde{V}_{u + 1}^k(s')))$. \label{alg-line:regressoin} \\
            $Q^k_u(\cdot, \cdot) \leftarrow \left(r(\cdot, \cdot) + \gamma (\langle \bm\varphi(\cdot, \cdot), \bm{w}^k_{u + 1} \rangle + \min_{s'} \widetilde{V}_{u + 1}^k(s') + \beta \Vert \bm\varphi(\cdot, \cdot) \Vert_{\Lambda_k^{-1}}) \right) \wedge \frac{1}{1 - \gamma}$. \label{alg-line:value-iteration-linear} \\
            $\widetilde{V}^k_u(\cdot) \leftarrow \max_a Q^k_u(\cdot, a)$. \\
            $V^k_u(\cdot) \leftarrow \widetilde{V}^k_u(\cdot) \wedge (\min_{s'} \widetilde{V}_u^k(s') + H)$. \label{alg-line:clipping-linear}
        }
    }
}
\caption{$\gamma$-LSCVI-UCB for linear MDP setting with minimum oracle}
\end{algorithm*}

\paragraph{Challenges}
Naively adapting the algorithm design and analysis for the tabular setting to the linear MDP setting would result in a regret bound that is polynomial in $S$, the size of the state space, when bounding $\sum_{t = 1}^T (V_{t - 1}(s_{t + 1}) - V_t(s_t))$.
Also, algorithmically making the state value function monotonically decrease in $t$ by taking minimum with the previous estimate every iteration, as is done in the tabular setting for the telescoping sum argument, would lead to an exponential covering number for the function class of the value function, in either $T$ or $S$ \parencite{he2023nearly}.
A major challenge in algorithm design and analysis is sidestepping these issues.
We now present our algorithm for the linear MDP setting, which addresses these issues.

\subsection{Algorithm}

Our algorithm, called \textit{discounted least-squares clipped value iteration with upper confidence bound} ($\gamma$-LSCVI-UCB), adapts LSVI-UCB \parencite{jin2020provably} developed for the episodic setting to the discounted setting.
We highlight key differences from LSVI-UCB below.

\paragraph{Clipping the Value Function}
We clip the value function estimates, as is done in the tabular setting in the previous section, to restrict the span (Line~\ref{alg-line:clipping-linear}), which saves a factor of $1/(1 - \gamma)$ in the regret bound.

\paragraph{Restricting the Range of Value Target}
When regressing $V_u^k(s')$ on $\bm\varphi(s, a)$ we subtract the value target by $\min_{s'} \widetilde{V}_u^k(s')$ and use $V_u^k(\cdot) - \min_{s'} \widetilde{V}_u^k(s')$ as the value target instead of $V_u^k(\cdot)$ (Line~\ref{alg-line:regressoin}).
This adjustment of the value target guarantees a bound on $\Vert \bm{w}_u^k \Vert_2$ that scales with the target span $H$ instead of $1 / (1 - \gamma)$, which is necessary for achieving $\widetilde{\mathcal{O}}(\sqrt{T})$ regret.
To compensate for the adjustment, we add back $\min_{s'} \widetilde{V}_u^k(s')$ when estimating the value target using the regression coefficient $\bm{w}_u^k$ : $\langle \bm\varphi(\cdot, \cdot), \bm{w}_u^k \rangle + \min_{s'} \widetilde{V}_u^k(s')$ (Line~\ref{alg-line:value-iteration-linear}).

In our previous algorithm $\gamma$-UCB-CVI, designed for the tabular setting, the value iteration step alternates with the decision making step.
At each time step $t$, a greedy action is selected based on the most recently constructed action value function $Q_t$.
This structure is common in value iteration based algorithms and $Q$-learning algorithms for both infinite-horizon average-reward tabular MDPs \parencite{zhang2023sharper} and infinite-horizon discounted tabular MDPs \parencite{liu2020regret,he2021nearly}.
With the coupling of value iteration and decision making steps, bounding the term $\sum_t V_{t - 1}(s_{t + 1}) - V_t(s_t)$ required enforcing $V_t$ to be decreasing in $t$ algorithmically since $V_t$ is one Bellman operation ahead of $V_{t - 1}$.

However, as discussed previously, in the linear MDP setting, forcing $V_t$ to be monotonically decreasing by taking the minimum with previous value functions would cause the log covering number of the function class for the value function to scale with $T$, making regret bound vacuous.
To sidestep this issue, we use a novel algorithm structure that decouples the value iteration step and the decision making step.

\paragraph{Planning until the End of Horizon}
Before taking any action at time $t$, we generate a sequence of action value functions $Q_T, Q_{T - 1}, \dots, Q_t$ by running $T - t$ value iterations (Line~\ref{alg-line:value-iteration-start}-\ref{alg-line:clipping-linear}).
Then, at each decision time step $t$, take a greedy action with respect to $Q_t$.
This algorithm structure is reminiscent of the value iteration based algorithms for the finite-horizon episodic setting~\parencite{azar2017minimax,jin2020provably}, where Bellman operations are performed to generate action value functions at each time step in an episode of fixed length, then greedy actions with respect to those action value functions are taken for the entire episode.
With the new algorithm structure, $Q_{t - 1}$ is now one Bellman operation ahead of $Q_t$, and the quantity of interest becomes $\sum_{t = 1}^T V_{t + 1}(s_{t + 1}) - V_t(s_t)$, which can be bounded by telescoping sum.

\paragraph{Restarting when Information Doubles}
If we generate all $T$ action value functions to be used at the initial time step by running approximate value iteration, and follow them for decision-making for $T$ steps, we cannot make use of the trajectory data collected.
To address this, and still use the scheme of pregenerating action value functions, we restart the process of running value iterations for $T$ steps every time a certain information measure of the collected data doubles.
This allows us to incorporate the newly collected trajectory data into subsequent decision-making.
We adopt the rarely-switching covariance matrix trick \parencite{wang2021provably}, which triggers a restart when the determinant of the empirical covariance matrix doubles (Line~\ref{alg-line:episode}).

Our algorithm has the following guarantee.

\begin{theorem} \label{theorem:linear-mdp}
Under Assumptions \ref{assumption:bellman-optimality} and \ref{assumption:linear-mdp}, running Algorithm~\ref{alg:lscvi-ucb} with inputs $\gamma = 1 - \sqrt{\log(T) / T}$, $\lambda = 1$, $H = 2 \cdot \text{sp}(v^\ast)$ and $\beta = 2 c_\beta \cdot \text{sp}(v^\ast) d \sqrt{\log(dT / \delta)}$ guarantees with probability at least $1 - \delta$,
$$
R_T \leq \mathcal{O}(\text{sp}(v^\ast) \sqrt{d^3 T \log(dT / \delta)}).
$$
\end{theorem}
The constant $c_\beta$ is an absolute constant defined in Lemma~\ref{lemma:concentration-regression}.
We expect careful analysis of the variance of the value estimate \parencite{he2023nearly} may improve our regret by a factor of $\sqrt{d}$. We leave this improvement to future work.

\subsection{Computational Complexity}

Our algorithm runs in episodes and since a new episode starts only when the determinant of the covariance matrix $\bar\Lambda_t$ doubles, there can be at most $\mathcal{O}(d \log_2 T)$ episodes (see Lemma~\ref{lemma:num-episode}).
In each episode, we run at most $T$ value iterations.
In each iteration step $u$, the algorithm computes $\min_{s'} \widetilde{V}_u^k(s')$ which requires evaluating $\widetilde{V}_u^k(s')$ at all $s' \in \mathcal{S}$, which requires $\mathcal{O}(d^2 S A)$ computations.
Also, the algorithm computes $\bm{w}_{u + 1}^k$, which requires $\mathcal{O}(d^2 + Td)$ operations.
All other operations runs in $\mathcal{O}(d^2 + A)$ per value iteration.
In total, the algorithm runs in $\mathcal{O}((\log_2 T) d^3 S A T^2)$.
See Appendix~\ref{section:computation} for detailed analysis.

The FOPO algorithm by \textcite{wei2021learning} that matches our regret bound under the same set of assumptions, has a time complexity of $\mathcal{O}(T^d \log_2 T)$.
Although our time complexity is an improvement over previous work in the sense that the time complexity is polynomial in problem parameters, it has linear dependency on $S$.
The dependency on $S$ arises from taking the minimum of value functions for clipping.
We conjecture that this dependency can be eliminated by using an estimate of the minimum rather than computing the global minimum of value functions.
For example, replacing $\min_{s'} V_u^k(s')$ with $\min_{s'} V^\ast(s')$ for clipping leads to the same regret bound (see Appendix~\ref{section:access-to-min}).
A promising approach is to use $\min_{s' \in \{s_1, \dots, s_t\}} \widetilde{V}_u^k(s')$, minimum over states visited so far, instead of the global minimum.
However, as discussed in Appendix~\ref{section:difficulty}, naively changing the clipping operation fails.
We leave eliminating the dependency on $S$ in the time complexity to future work.

\subsection{Analysis}

In this section, we outline the proof of the regret bound presented in Theorem~\ref{theorem:linear-mdp}.
We first show that the value iteration step in Line~\ref{alg-line:value-iteration-linear} with the bonus term $\beta \Vert \bm\phi(\cdot, \cdot) \Vert_{\Lambda_k^{-1}}$ with appropriately chosen $\beta$ ensures the value function estimates $V_t$ and $Q_t$ to be optimistic estimates of $V^\ast$ and $Q^\ast$, respectively.
The argument is based on the following concentration inequality for the regression coefficients.
See Appendix~\ref{appendix:concentration-linear} for a proof.

\begin{lemma}[Concentration of regression coefficients]
\label{lemma:concentration-regression}
With probability at least $1 - \delta$, there exists an absolute constant $c_\beta$ such that for $\beta = c_\beta \cdot H d \sqrt{\log(dT / \delta)}$, we have
$$
\vert \langle \bm{\phi}, \bm{w}_u^k - \bm{w}_u^k{}^\ast \rangle \vert \leq \beta \Vert \bm{\phi} \Vert_{\Lambda_{k}^{-1}}
$$
for all episode indices $k$ and for all vectors $\bm{\phi} \in \mathbb{R}^d$ where $\bm{w}_u^k{}^\ast \coloneqq \int (V_u^k(s) - \min_{s'} V_u^k(s')) d\bm\mu(s)$ is a parameter that satisfies $\langle \bm\varphi(s, a), \bm{w}_u^k{}^\ast \rangle = [PV_u^k](s, a) - \min_{s'} V_u^k(s')$.
\end{lemma}

With the concentration inequality, we can show the following optimism result.
See Appendix~\ref{appendix:optimism-linear} for an induction-based proof.

\begin{lemma}[Optimism] \label{lemma:optimism-linear}
Under the linear MDP setting, running Algorithm~\ref{alg:lscvi-ucb} with input $H = 2 \cdot \text{sp}(v^\ast)$ guarantees with probability at least $1 - \delta$ that
for all episodes $k = 1, 2, \dots$, $u = t_k, \dots, T + 1$ and for all $(s, a) \in \mathcal{S} \times \mathcal{A}$, we have
$$
V_u^k(s) \geq V^\ast(s), \quad Q_u^k(s, a) \geq Q^\ast(s, a).
$$
\end{lemma}

Now, we show the regret bound under the event that the high probability events in the previous two lemmas (Lemma~\ref{lemma:concentration-regression}, Lemma~\ref{lemma:optimism-linear}) hold.
Let $t$ be a time step in episode $k$ such that both $t$ and $t + 1$ are in episode $k$.
By the definition of $Q^k_u(\cdot, \cdot)$ (Line~\ref{alg-line:value-iteration-linear}), we have for all $t = t_k, \dots, T + 1$ and $(s, a) \in \mathcal{S} \times \mathcal{A}$ that
\begin{align*}
r(s, a)
&\geq Q^k_t(s, a) - \gamma ( \langle \bm\varphi(s, a), \bm{w}^k_{t + 1} \rangle + \min_{s'} V^k_{t + 1}(s') - \beta \Vert \bm\varphi(s, a) \Vert_{\Lambda_k^{-1}}) \\
&\geq
Q^k_t(s, a) - \gamma [PV^k_{t + 1}](s, a) - 4 \beta \Vert \bm\varphi(s, a) \Vert_{\bar\Lambda_t^{-1}}
\end{align*}
where the second inequality uses the concentration bound for the regression coefficients in Lemma~\ref{lemma:concentration-regression}.
It also uses $\Vert \bm{x} \Vert_{\Lambda_k^{-1}} \leq 2 \Vert \bm{x} \Vert_{\Lambda_t^{-1}}$ (Lemma~\ref{lemma:doubling}).
Hence, we can bound the regret in episode $k$ by
\begin{align*}
R^k &= \sum_{t = t_k}^{t_{k + 1} - 1} (J^\ast - r(s_t, a_t)) \\
&\leq \sum_{t = t_k}^{t_{k + 1} - 1} (J^\ast - Q^k_t(s_t, a_t) + \gamma [PV_{t + 1}^k](s_t, a_t) + 4 \beta \Vert \bm\varphi(s_t, a_t) \Vert_{\bar\Lambda_t^{-1}} ),
\end{align*}
which can be decomposed into
\begin{align*}
&=
\underbrace{ \sum_{t = t_k}^{t_{k + 1} - 1} (J^\ast - (1 - \gamma) V^k_{t + 1}(s_{t + 1}))}_{(a)} +
\gamma \underbrace{ \sum_{t = t_k}^{t_{k + 1} - 1} (V^k_{t + 1}(s_{t + 1}) - Q^k_t(s_t, a_t))}_{(b)} \\
&\hspace{20mm}+ \gamma \underbrace{\sum_{t = t_k}^{t_{k + 1} - 1} [PV^k_{t + 1}](s_t, a_t) - V^k_{t + 1}(s_{t + 1}))}_{(c)} + 4 \beta \underbrace{\sum_{t = t_k}^{t_{k + 1} - 1} \Vert \bm\varphi(s_t, a_t) \Vert_{\bar\Lambda_{t}^{-1}}}_{(d)}
\end{align*}
where the first inequality uses the bound for $r(s_t, a_t)$.
With the same argument as in the tabular case, the term $(a)$ summed over all episodes can be bounded by $T(1 - \gamma) \text{sp}(v^\ast)$ using the optimism $V^k_u(s_{t + 1}) \geq V^\ast(s_{t + 1})$, and Lemma~\ref{lemma:discounted-approximation} that bounds $\vert J^\ast - (1 - \gamma) V^\ast(s) \vert$ for all $s \in \mathcal{S}$.
Term $(d)$, summed over all episodes, can be bounded by $\mathcal{O}(\beta \sqrt{dT \log T})$ using Cauchy-Schwartz and Lemma~\ref{lemma:bonus-term-linear}.
Term $(c)$, summed over all episodes, is a sum of a martingale difference sequence, which can be bounded by $\mathcal{O}(\text{sp}(v^\ast) \sqrt{T \log(1 / \delta)})$ since $\text{sp}(V^k_u) \leq 2 \cdot \text{sp}(v^\ast)$ by the clipping step in Line~\ref{alg-line:clipping-linear}.
\paragraph{Bounding Term $\bm{(b)}$}
To bound term $(b)$ note that
\begin{align*}
V^k_{t + 1}(s_{t + 1})
&\leq \widetilde{V}^k_{t + 1}(s_{t + 1}) \\
&= \max_a Q^k_{t + 1}(s_{t + 1}, a) \\
&= Q^k_{t + 1}(s_{t + 1}, a_{t + 1})
\end{align*}
as long as the time step $t + 1$ is in episode $k$, since the algorithm chooses $a_{t + 1}$ that maximizes $Q^k_{t + 1}(s_{t + 1}, \cdot)$.
Hence,
\begin{align*}
\sum_{t = t_k}^{t_{k + 1} - 1} (V^k_{t + 1}(s_{t + 1}) - Q^k_t(s_t, a_t))
&\leq \frac{1}{1 - \gamma} + \sum_{t = t_k}^{t_{k + 1} - 2} (Q^k_{t + 1}(s_{t + 1}, a_{t + 1}) - Q^k_t(s_t, a_t)) \\
&\leq \mathcal{O}\left(\frac{1}{1 - \gamma}\right)
\end{align*}
where the second inequality uses telescoping sum and the fact that $Q^k_t \leq \frac{1}{1 - \gamma}$.
Since the episode is advanced when the determinant of the covariance matrix doubles, it can be shown that the number of episodes is bounded by $\mathcal{O}(d \log (T))$ (Lemma~\ref{lemma:num-episode}).
Combining all the bounds, and using $\beta = \mathcal{O}(\text{sp}(v^\ast) d \sqrt{\log(dT / \delta)}$, we get
\begin{align*}
R_T &\leq \mathcal{O}\Big(T(1 - \gamma) \text{sp}(v^\ast) + \frac{d}{1 - \gamma} \log(T) + \text{sp}(v^\ast) \sqrt{T \log(1 / \delta)} + \text{sp}(v^\ast) \sqrt{d^3 T \log(dT / \delta)} \Big).
\end{align*}
Setting $\gamma = 1 - \sqrt{(\log T) / T}$, we get
$$
R_T \leq \mathcal{O}\left(
\text{sp}(v^\ast) \sqrt{d^3 T \log(dT / \delta)}
\right),
$$
which concludes the proof of Theorem~\ref{theorem:linear-mdp}.

\section{Conclusion}
In this paper, we propose an algorithm with time complexity polynomial in the problem parameters that achieves $\widetilde{\mathcal{O}}(\sqrt{T})$ regret for infinite-horizon average-reward linear MDPs without making a strong ergodicity assumption on the dynamics.
Our algorithm approximates the average-reward setting by the discounted setting with a carefully tuned discounting factor.
A key technique that allows for order optimal regret bound is bounding the span of the value function in each value iteration step via clipping.
Additionally, we precompute a sequence of action value functions by running value iterations, then use them in reverse order for taking actions.
Eliminating the dependence on the size of the state space in time complexity remains an open problem.
Another promising direction for future work would be to extend these methods to the general function approximation setting.

\section{Acknowledgements}

KH and AT acknowledge the support of the National Science Foundation (NSF) under grant IIS-2007055. DL acknowledges the support of the National Research Foundation of Korea (NRF) under grants No. RS-2024-00350703 and No. RS-2024-00410082.

\section*{References}
\printbibliography[heading=none]

\newpage

\appendix

\section{Tabular Setting} \label{section:analysis-tabular}

Central to the analysis of the concentration bound for the approximate Bellman backup is the following concentration bound for scalar-valued self-normalized processes.
\begin{lemma}[Concentration of Scalar-Valued Self-Normalized Processes \label{lemma:scalar-process-concentration}
\parencite{abbasi2012online}]
Let $\{\varepsilon_t\}_{t = 1}^\infty$ be a real-valued stochastic process with corresponding filtration $\{\mathcal{F}_t \}_{t = 0}^\infty$.
Let $\varepsilon_t | \mathcal{F}_{t - 1}$ be zero-mean and $\sigma$-subgaussian.
Let $\{Z_t \}_{t = 0}^\infty$ be an $\mathbb{R}$-valued stochastic process where $Z_t \in \mathcal{F}_{t - 1}$.
Assume $W > 0$ is deterministic.
Then for any $\delta > 0$, with probability at least $1 - \delta$, we have for all $t \geq 0$ that
$$
\frac{( \sum_{s = 1}^t Z_s \varepsilon_s)^2}{W + \sum_{s = 1}^t Z_s^2} \leq 2 \sigma^2 \log \left( \frac{\sqrt{W + \sum_{s = 1}^t Z_s^2}}{\delta \sqrt{W}} \right).
$$
\end{lemma}

\subsection{Proof of Lemma~\ref{lemma:model-error-tabular}} \label{section:model-error-tabular}

To show a bound for $\vert (\widehat{P}_t - P) V_t(s, a) \vert$ uniformly on $t \in [T]$, we use a covering argument on the function class that captures $V_t$.
Note that the value functions $V_t$ defined in the algorithm always lie in the following function class.
$$
\mathcal{V}_{\text{tabular}} = \{ v \in \mathbb{R}^\mathcal{S} : v(s) \in [0, \textstyle \frac{1}{1 - \gamma}] ~\text{for all}~ s \in \mathcal{S} \}.
$$
We first bound the error for a fixed value function in $\mathcal{V}_\text{tabular}$.
Afterward, we will use a covering argument to get a uniform bound over $\mathcal{V}_{\text{tabular}}$.
\begin{lemma} \label{lemma:fixed-value-error-tabular}
 Fix any $V\in\mathcal{V}_\text{tabular}$. There exists  some constant $C$ such that for any $\delta\in(0,1)$, with probability at least $1-\delta$, we have:
$$
\vert [(\widehat{P}_t - P)V](s, a)\vert\le C\text{sp}(v^\ast)\sqrt{\frac{\log(SAT/\delta)}{N_t(s,a)}}
$$
for all $(s,a)\in \mathcal{S}\times \mathcal{A}$ and $t=1,\dots,T$.
\end{lemma}

\begin{proof}

Fix any $(s,a)\in \mathcal{S}\times\mathcal{A}$. By definition, we have: 
$$
[(\widehat{P}_t - P)V](s, a) = \frac{1}{N_t(s, a)}\sum_{\tau = 1}^t \mathbb{I}\{s_\tau=s, a_\tau = a\} [V(s_{\tau + 1}) - [P V](s, a)].
$$

Let $\varepsilon_t = V(s_{t + 1}) - [PV](s_t, a_t)$, $Z_t = \mathbb{I}\{ s_t = s, a_t = a\}$, and $W=1$.
Since the range of $\varepsilon_t$ is bounded by $2\cdot \text{sp}(v^\ast)$, it is $\text{sp}(v^\ast)$-subgaussian. By Lemma \ref{lemma:scalar-process-concentration}, we know for some constant $C$, with probability at least $1 - \delta$,  for all $t=1,\dots, T$, we have
\begin{align*}
\vert [(\widehat{P}_t - P)V](s, a)\vert
&= \frac{\vert \sum_{s=1}^t Z_s\varepsilon_s \vert}{1 + \sum_{s=1}^t Z_s^2} \\
&\le C \cdot \text{sp}(v^\ast)\sqrt{\frac{\log(\sqrt{N_t(s,a)}/\delta)}{N_t(s,a)}} \\
&\le C \cdot \text{sp}(v^\ast)\sqrt{\frac{\log(T/\delta)}{N_t(s,a)}}.
\end{align*}

Applying a union bound for all $(s,a)\in\mathcal{S}\times \mathcal{A}$ gives us the desired inequality.

\end{proof}

We use $\mathcal{N_\epsilon}$ to denote the $\epsilon$-covering number of $\mathcal{V}_\text{tabular}$ with respect to the distance $\text{dist}(V,V')=\Vert V-V'\Vert_\infty$.
Using a grid of size $\epsilon$, since functions in $\mathcal{V}_{\text{tabular}}$ has the range $[0, \frac{1}{1 - \gamma}]$, it can be seen that $\log \mathcal{N}_\epsilon\le S\log \frac{1}{\epsilon(1-\gamma)}$.
Now, we prove a uniform concentration bound using a covering argument on $\mathcal{V}_{\text{tabular}}$.

\begin{proof}[Proof of Lemma~\ref{lemma:model-error-tabular}]
Note that $V_t \in \mathcal{V}_{\text{tabular}}$ for all $t$.
Consider an $\epsilon$-cover of $\mathcal{V}_{\text{tabular}}$.
For any $V_t \in \mathcal{V}_{\text{tabular}}$, there exists $\widetilde{V}_t$ in the $\epsilon$-cover such that $\sup_s \vert V_t(s) - \widetilde{V}_t(s) \vert \leq \epsilon$.
Thus, we have 
$$
\vert [(\widehat{P}_t - P) V_t](s, a) \vert \le \vert [(\widehat{P}_t - P) \widetilde{V}_t](s, a) \vert + \vert [(\widehat{P}_t - P) (V_t - \widetilde{V}_t)](s, a) \vert \le \vert [(\widehat{P}_t - P) \widetilde V](s, a) \vert + 2\epsilon.
$$

We then apply Lemma \ref{lemma:fixed-value-error-tabular} and a union bound to obtain:
$$
\begin{aligned}
\vert [(\widehat{P}_t - P) V_t](s, a) \vert
&\le C \cdot \text{sp}(v^\ast)\sqrt{\frac{\log(SAT\mathcal{N}_\epsilon/\delta)}{N_t(s,a)}}+2\epsilon \\
&\le C \cdot \text{sp}(v^\ast)\sqrt{\frac{ \log (SAT / \delta) + S \log (1 / (\epsilon (1 - \gamma)))}{N_t(s,a)}}+2\epsilon.
\end{aligned}
$$

Picking $\epsilon=\frac{1}{\sqrt{T}}$ concludes the proof.
\end{proof}

\subsection{Proof of Lemma~\ref{lemma:optimism-tabular}} \label{section:optimism-tabular}

\begin{proof}[Proof of Lemma~\ref{lemma:optimism-tabular}]
We prove by induction on $t \geq 1$.    
The base case $t = 1$ is trivial since Algorithm~\ref{alg:ucbvi} initializes $V_1(\cdot) = \frac{1}{1 - \gamma}$, $Q_1(\cdot, \cdot) = \frac{1}{1 - \gamma}$.
Now, suppose $V_1, \dots, V_t \geq V^\ast$ and $Q_1, \dots, Q_t \geq Q^\ast$.

We first show that $Q_{t+1}(s,a)\ge Q^\ast(s,a)$. By the Bellman optimality equation for the discounted setting, we have for all $(s, a) \in \mathcal{S} \times \mathcal{A}$ that
$$
Q^\ast(s, a) = r(s, a) + \gamma [PV^\ast](s, a).
$$
Fix any pair $(s, a) \in \mathcal{S} \times \mathcal{A}$.
By the definition of $Q_{t + 1}$ in Line~\ref{alg-line:q} of Algorithm~\ref{alg:ucbvi}, we have
\begin{align*}
Q_{t + 1}(s, a)
&= (r(s, a) + \gamma [\widehat{P}_t V_t](s, a) + \beta / \sqrt{N_t(s, a)}) \wedge Q_t(s, a) \\
&\geq (r(s, a) + \gamma [P V_t](s, a)) \wedge Q_t(s, a) \\
&\geq (r(s, a) + \gamma [P V^\ast](s, a)) \wedge Q^\ast(s, a) \\
&= Q^\ast(s, a)
\end{align*}
where the first inequality is by the concentration inequality in Lemma~\ref{lemma:model-error-tabular} and our choice of $\beta$ in Theorem~\ref{thm:tabular}, and the second inequality is by the induction hypotheses $V_t \geq V^\ast$ and $Q_t \geq Q^\ast$.
The last equality is by the Bellman optimality equation.

Now, we show $V_{t + 1}(s) \geq V^\ast(s)$.
By the definition of $\widetilde{V}_{t+1}$ in Line \ref{alg-line:v-tilde} of Algorithm \ref{alg:ucbvi}, we have
$$
\begin{aligned}
\widetilde{V}_{t + 1}(s)
&= (\max_a Q_{t + 1}(s, a)) \wedge V_t(s) \\
&\geq (\max_a Q^\ast(s, a)) \wedge V^\ast(s) \\
&= V^\ast(s)
\end{aligned}
$$
where the inequality is by the optimism of $Q_{t + 1}$ we just proved, and the induction hypothesis $V_t \geq V^\ast$.

Finally, by the definition of $V_{t+1}$ in Line \ref{alg-line:projection} of Algorithm \ref{alg:ucbvi}, we have
$$
V_{t + 1}(s) = \widetilde V_{t+1}(s)\wedge (\min_{s'}\widetilde V_{t+1}(s')+\text{sp}(v^\ast)) \geq \widetilde V_{t+1}(s)\ge V^\ast(s),
$$
which completes the proof by induction.
\end{proof}

\subsection{Omitted Proofs}

\begin{lemma}[Sum of Bonus Terms] \label{lemma:sum-of-bonus-terms-tabular}
Consider running Algorithm~\ref{alg:ucbvi}.
The sum of the bonus terms can be bounded by
$$
\sum_{t = 1}^T \sqrt{1 / N_{t - 1}(s_t, a_t)} \leq 2 \sqrt{SAT}.
$$
\end{lemma}
\begin{proof}
Recall that $N_t(s, a) = 1 + \sum_{s = 1}^t \mathbb{I}\{ s_t = s, a_t = a \}$.
For convenience, write $n_t(s, a) = \sum_{s = 1}^t \mathbb{I} \{ s_t = s, a_t = a\}$ such that $N_t(s, a) = 1 + n_t(s, a)$.
Then,
\begin{align*}
\sum_{t = 1}^T\sqrt{1 / N_{t - 1}(s_t, a_t)}
&\leq \sum_{s \in \mathcal{S}} \sum_{a \in \mathcal{A}} \sum_{n = 1}^{n_T(s, a)} \sqrt{1 / n} \\
&\leq 2 \sum_{s \in \mathcal{S}} \sum_{a \in \mathcal{A}} \sqrt{n_T(s, a)} \\
&\leq 2 \sqrt{SA T}
\end{align*}
where the second inequality uses the identity $\sum_{n = 1}^N 1 / \sqrt{n} \leq 2 \sqrt{N}$, and the last inequality is by Cauchy-Schwarz and the fact that $\sum_s \sum_a n_T(s, a) = T$.
\end{proof}

\section{Linear MDP Setting} \label{appendix:linear-mdp}

Central to the analysis of the concentration bound for the approximate Bellman backup is the following concentration bound for scalar-valued self-normalized processes.
\begin{lemma}[Concentration of vector-valued self-normalized processes \parencite{abbasi2011improved}]
Let $\{\varepsilon_t\}_{t = 1}^\infty$ be a real-valued stochastic process with corresponding filtration $\{\mathcal{F}_t \}_{t = 0}^\infty$.
Let $\varepsilon_t | \mathcal{F}_{t - 1}$ be zero-mean and $\sigma$-subgaussian.
Let $\{\phi_t \}_{t = 0}^\infty$ be an $\mathbb{R}^d$-valued stochastic process where $\phi_t \in \mathcal{F}_{t - 1}$.
Assume $\Lambda_0$ is a $d \times d$ positive definite matrix, and let $\Lambda_t = \Lambda_0 + \sum_{s = 1}^t \phi_s \phi_s^T$.
Then for any $\delta > 0$, with probability at least $1 - \delta$, we have for all $t \geq 0$ that
$$
\left\Vert \sum_{s = 1}^t \phi_s \varepsilon_s \right\Vert_{\Lambda_t^{-1}}^2 \leq 2 \sigma^2 \log \left( \frac{\text{det}(\Lambda_t)^{1/2} \text{det}(\Lambda_0)^{-1/2}}{\delta} \right).
$$
\end{lemma}

\subsection{Concentration Bound for Regression Coefficients} \label{appendix:concentration-linear}

\begin{lemma} \label{lemma:pv}
Let $V : \mathcal{S} \rightarrow [0, B]$ be a bounded function.
Then, there exists a parameter $\bm{w}_V^\ast \in \mathbb{R}^d$ such that $PV(s, a) = \langle \varphi(s, a), \bm{w}_V \rangle$ for all $(s, a) \in \mathcal{S} \times \mathcal{A}$ and
$$
\Vert \bm{w}_V^\ast \Vert_2 \leq B\sqrt{d}.
$$
\end{lemma}
\begin{proof}
By Assumption \ref{assumption:linear-mdp}, we have
$$
[PV](s, a) = \int_\mathcal{S} V(s') P(ds' | s, a)
= \int_\mathcal{S} V(s') \langle \bm\varphi(s, a), \bm\mu(ds') \rangle
= \langle \bm\varphi(s, a), \int_\mathcal{S} V(s') d \bm\mu(s') \rangle.
$$
Hence, $\bm{w}_V^\ast = \int_\mathcal{S} V(s') d\bm\mu(s')$ satisfies $[PV](s, a) = \langle \bm\varphi(s, a), \bm{w}_V \rangle$ for all $(s, a) \in \mathcal{S} \times \mathcal{A}$.
Also, such $\bm{w}_V$ satisfies
$$
\Vert \bm{w}_V \Vert_2 = \left\Vert \int_\mathcal{S} V(s') d \bm\mu(s') \right\Vert_2
\leq B \left\Vert \int_\mathcal{S} d \bm\mu(s') \right\Vert_2 \leq B \sqrt{d}
$$
where the first inequality holds since $\bm\mu$ is a vector of positive measures and $V(s') \geq 0$.
The last inequality is by the boundedness assumption \eqref{eqn:boundedness} on $\bm\mu(\mathcal{S})$.
\end{proof}

\begin{lemma} \label{lemma:bound-weight-algorithm}
Let $\bm{w}$ be a ridge regression coefficient obtained by regressing $y \in [0, B]$ on $\bm{x} \in \mathbb{R}^d$ using the dataset $\{ (\bm{x}_i, y_i) \}_{i = 1}^n$ so that $\bm{w} = \Lambda^{-1} \sum_{i = 1}^n \bm{x}_i y_i$
where $\Lambda = \sum_{i = 1}^n \bm{x}\bm{x}^T + \lambda I$.
Then,
$$
\Vert \bm{w} \Vert_2 \leq B \sqrt{dn / \lambda}.
$$
\end{lemma}
\begin{proof}
For any unit vector $\bm{u} \in \mathbb{R}^d$ with $\Vert \bm{u} \Vert_2 = 1$, we have
\begin{align*}
\vert \bm{u}^T \bm{w} \vert
&= \left\vert \bm{u}^T \Lambda^{-1} \sum_{i = 1}^n \bm{x}_i y_i \right\vert \\
&\leq B \sum_{i = 1}^n \vert \bm{u}^T \Lambda^{-1} \bm{x}_i \vert \\
&\leq B \sum_{i = 1}^n \sqrt{\bm{u}^T \Lambda^{-1} \bm{u}} \sqrt{\bm{x}_i^T \Lambda^{-1} \bm{x}_i}  \\
&\leq \frac{B}{\sqrt{\lambda}} \sum_{i = 1}^n \sqrt{\bm{x}_i^T \Lambda^{-1} \bm{x}_i} \\
&\leq \frac{B}{\sqrt{\lambda}} \sqrt{n} \sqrt{\sum_{i = 1}^n {\bm{x}_i^T \Lambda^{-1} \bm{x}_i}} \\
&\leq B\sqrt{dn / \lambda}
\end{align*}
where the second inequality and the fourth inequality are by Cauchy-Schwartz, the third inequality is by $\Lambda \succeq \lambda I$, and the last inequality is by Lemma~\ref{lemma:fixed-bonus-sum}.

The desired result follows from the fact that $\Vert \bm{w} \Vert_2 = \max_{\bm{u} : \Vert \bm{u} \Vert_2 = 1} \vert \bm{u}^T \bm{w} \vert$.
\end{proof}

The following self-normalized process bound is an adaptation of Lemma D.4 in \textcite{jin2020provably}.
Their proof defines the bound $B$ to be a value that satisfies $\Vert V \Vert_\infty \leq B$.
Upon observing their proof, it is easy to see that we can strengthen their result to require only $\text{sp}(V) \leq B$.
The following lemma is the strengthened version.
\begin{lemma}[Adaptation of Lemma D.4 in \textcite{jin2020provably}]
\label{lemma:self-normalized-process-bound}
Let $\{ x_t \}_{t = 1}^\infty$ be a stochastic process on state space $\mathcal{S}$ with corresponding filtration $\{ \mathcal{F}_t \}_{t = 0}^\infty$.
Let $\{ \phi_t \}_{t = 0}^\infty$ be a $\mathbb{R}^d$-valued stochastic process where $\phi_t \in \mathcal{F}_{t - 1}$, and $\Vert \phi_t \Vert_2 \leq 1$.
Let $\Lambda_n = \lambda I + \sum_{t = 1}^n \phi_t \phi_t^T$.
Then for any $\delta > 0$ and any given function class $\mathcal{V}$, with probability at least $1 - \delta$, for all $n \geq 0$, and any $V \in \mathcal{V}$ satisfying $\text{sp}(V) \leq H$, we have
$$
\left\Vert \sum_{t = 1}^n \phi_t (V(x_t) - \mathbb{E}[V(x_t) | \mathcal{F}_{t - 1}]) \right\Vert_{\Lambda_n^{-1}}^2
\leq
4 H^2 \left[\frac{d}{2} \log\left(\frac{n + \lambda}{\lambda} \right) + \log \frac{\mathcal{N}_\varepsilon}{\delta}\right] + \frac{8 n^2 \varepsilon^2}{\lambda}
$$
where $\mathcal{N}_\varepsilon$ is the $\varepsilon$-covering number of $\mathcal{V}$ with respect to the distance $\text{dist}(V, V') = \sup_x \vert V(x) - V'(x) \vert$.
\end{lemma}

\begin{lemma}[Adaptation of Lemma D.6 in \textcite{jin2020provably}]
\label{lemma:covering-number-linear}
Let $\mathcal{V}_\text{linear}$ be a class of functions mapping from $\mathcal{S}$ to $\mathbb{R}$ with the following parametric form
\begin{equation} \label{eqn:v-linear}
V(\cdot) = (\max_a \bm{w}^T \bm\varphi(\cdot, a) + v + \beta \sqrt{\bm\varphi(\cdot, a)^T \Lambda^{-1} \bm\varphi(\cdot, a)}) \wedge M
\end{equation}
where the parameters $(\bm{w}, \beta, v, \Lambda, M)$ satisfy $\Vert \bm{w} \Vert \leq L$, $\beta \in [0, B]$, $v \in [0, D]$, $M \geq 0$ and the minimum eigenvalue satisfies $\lambda_{\text{min}} (\Lambda) \geq \lambda$.
Assume $\Vert \bm\varphi(s, a) \Vert \leq 1$ for all $(s, a)$ pairs, and let $\mathcal{N}_\varepsilon$ be the $\varepsilon$-covering number of $\mathcal{V}$ with respect to the distance $\text{dist}(V, V') = \sup_x \vert V(x) - V'(x) \vert$.
Then
$$
\log \mathcal{N}_\varepsilon \leq d \log(1 + 8L / \varepsilon) + \log(1 + 4D / \varepsilon) + d^2 \log[ 1 + 8 d^{1/2} B^2 / (\lambda \varepsilon^2)].
$$
\end{lemma}

For the next lemma, we define value functions $V_u^{(t)}$ to be the functions obtained by the following value iteration (analogous to Line~\ref{alg-line:value-iteration-start}-\ref{alg-line:clipping-linear} in Algorithm~\ref{alg:lscvi-ucb}):

{\RestyleAlgo{plain} \SetAlgoCaptionLayout{} \LinesNumberedHidden
\begin{algorithm}
\LinesNotNumbered
$V^{(t)}_{T + 1}(\cdot) \leftarrow \frac{1}{1 - \gamma}$, $V^{(t)}_{T + 1}(\cdot) \leftarrow \frac{1}{1 - \gamma}$. \\
\For{$u = T, T - 1, \dots, 1$}{
     $\bm{w}^{(t)}_{u + 1} \leftarrow \bar\Lambda_t^{-1} \sum_{\tau = 1}^t \bm\varphi(s_\tau, a_\tau) (V^{(t)}_{u + 1}(s_{\tau + 1}) - \min_{s'}\widetilde{V}^{(t)}_{u + 1}(s'))$. \\
     $\widetilde{Q}^{(t)}_u(\cdot, \cdot) \leftarrow \left(r(\cdot, \cdot) + \gamma (\langle \bm\varphi(\cdot, \cdot), \bm{w}^{(t)}_{u + 1} \rangle + \min_{s'}\widetilde{V}^{(t)}_{u + 1}(s') + \beta \Vert \bm\varphi(\cdot, \cdot) \Vert_{\bar\Lambda_t^{-1}}) \right) \wedge \frac{1}{1 - \gamma}$. \\
     $\widetilde{V}^{(t)}_u(\cdot) \leftarrow \max_a \widetilde{Q}^{(t)}_u(\cdot, a)$. \\
     $V^{(t)}_u(\cdot) \leftarrow \widetilde{V}^{(t)}_u(\cdot) \wedge ( \min_{s'} \widetilde{V}^{(t)}_u(s') + H)$.
}
\end{algorithm}
}

With this definition, we show a high-probability bound on $\Vert \sum_{\tau = 1}^t \bm\varphi(s_\tau, a_\tau) [ V_u^{(t)} (s_{\tau + 1}) - PV_u^{(t)} (s_\tau, a_\tau)] \Vert_{\bar\Lambda_t^{-1}}$ uniformly on $u \in [T]$ and $t \in [T]$.
Since the tuple $(t_k - 1, V_u^k, \Lambda_k)$ encountered in Algorithm~\ref{alg:lscvi-ucb} is the same as the pair $(t, V_u^{(t)}, \bar\Lambda_t)$ for some $t \in [T]$, the uniform bound implies bound on $\Vert \sum_{\tau = 1}^{t_k - 1} \bm\varphi(s_\tau, a_\tau) [ V_u^k (s_{\tau + 1}) - PV_u^k(s_\tau, a_\tau) ] \Vert_{\Lambda_k^{-1}}$ for all episode $k$.

\begin{lemma}[Adaptation of Lemma B.3 in \textcite{jin2020provably}] \label{lemma:concentration-inequality}
Under the linear MDP setting in Theorem~\ref{theorem:linear-mdp} for the $\gamma$-LSCVI-UCB algorithm with clipping oracle (Algorithm~\ref{alg:lscvi-ucb}), let $c_\beta$ be the constant in the definition of $\beta = c_\beta H d \sqrt{\log(dT / \delta)}$.
There exists an absolute constant $C$ that is independent of $c_\beta$ such that for any fixed $\delta \in (0, 1)$, the event $\mathcal{E}$ defined by
\begin{align*}
\forall u \in [T],~&t \in [T] : \\
&\hspace{2mm} \left\Vert \sum_{\tau = 1}^t \bm\varphi(s_\tau, a_\tau) [V_u^{(t)}(s_{\tau + 1}) - [P V_u^{(t)}](s_\tau, a_\tau)] \right\Vert_{\bar\Lambda_t^{-1}} \leq C \cdot H d \sqrt{\log((c_\beta + 1) d T / \delta)}
\end{align*}
satisfies $P(\mathcal{E}) \geq 1 - \delta$.
\end{lemma}
\begin{proof}
For all $t = 1, \dots, T$, by Lemma~\ref{lemma:bound-weight-algorithm}, we have $\Vert \bm{w}_t \Vert_2 \leq H \sqrt{dt / \lambda}$.
Hence, by combining Lemma~\ref{lemma:covering-number-linear} and Lemma~\ref{lemma:self-normalized-process-bound}, for any $\varepsilon > 0$ and any fixed pair $(u, t) \in [T] \times [T]$, we have with probability at least $1 - \delta / T^2$ that
\begin{align*}
&\bigg\Vert \sum_{\tau = 1}^t \bm\varphi(s_\tau, a_\tau) [V_u^{(t)}(s_{\tau + 1}) - [P V_u^{(t)}](s_\tau, a_\tau)] \bigg\Vert_{\bar\Lambda_t^{-1}}^2 \\
&\hspace{20mm}\leq
4 H^2 \left[
\frac{2}{d} \log \left( \frac{t + \lambda}{\lambda} \right)
+ d \log \left(1 + \frac{4 H \sqrt{dt}}{\varepsilon \sqrt{\lambda}} \right)
+ d^2 \log \left( 1 + \frac{8 d^{1/2} \beta^2}{\varepsilon^2 \lambda} \right) + \log\left( \frac{T^2}{\delta} \right)
\right] + \frac{8 t^2 \varepsilon^2}{\lambda}
\end{align*}
where we use the fact that $\tau_t \leq t$.
Using a union bound over $(u, t) \in [T] \times [T]$ and choosing $\varepsilon = H d / t$ and $\lambda = 1$, there exists an absolute constant $C > 0$ independent of $c_\beta$ such that, with probability at least $1 - \delta$,
$$
\bigg\Vert \sum_{\tau = 1}^{t} \varphi(s_\tau, a_\tau) [V_u^{(t)}(s_{\tau + 1}) - [P V_u^{(t)}](s_\tau, a_\tau)] \bigg\Vert_{\bar\Lambda_t^{-1}}^2
\leq
C^2 \cdot d^2 H^2 \log((c_\beta + 1) d T / \delta),
$$
which concludes the proof.
\end{proof}

\begin{proof}[Proof of Lemma~\ref{lemma:concentration-regression}]
We prove under the event $\mathcal{E}$ defined in Lemma~\ref{lemma:concentration-inequality}.
For convenience, we introduce the notation $\bar{V}_u^k(s) = V_u^k(s) - \min_{s'} V_u^k(s')$.
With this notation, we can write
$$
\bm{w}_u^k = \Lambda_{k}^{-1} \sum_{\tau = 1}^{t_k - 1} \bm\varphi(s_\tau, a_\tau) \bar{V}_u^k(s_{\tau + 1}).
$$
We can decompose $\langle \bm{\phi}, \bm{w}_u^k \rangle$ as
\begin{align*}
\langle \bm{\phi}, \bm{w}_u^k \rangle
&= \underbrace{\langle \bm{\phi}, \Lambda_{k}^{-1} \sum_{\tau = 1}^{t_k - 1} \bm\varphi(s_\tau, a_\tau) [P\bar{V}_u^k](s_\tau, a_\tau) \rangle}_{(a)}
+ \underbrace{\langle \bm{\phi}, \Lambda_{k}^{-1} \sum_{\tau = 1}^{t_k - 1} \bm\varphi(s_\tau, a_\tau) (\bar{V}_u^k(s_{\tau + 1}) - P\bar{V}_u^k(s_\tau, a_\tau))}_{(b)}.
\end{align*}
Since $\bm{w}_u^k{}^\ast = \int \bar{V}_u^k(s) d \bm\mu(s)$ and $\bar{V}_u^k(s) \in [0, H]$ for all $s \in \mathcal{S}$, it follows by Lemma~\ref{lemma:pv} that $\Vert \bm{w}_u^k{}^\ast \Vert_2 \leq H \sqrt{d}$.
Hence, the first term $(a)$ in the display above can be bounded as
\begin{align*}
\langle \bm{\phi}, \Lambda_{k}^{-1} \sum_{\tau = 1}^{t_k - 1} \bm\varphi(s_\tau, a_\tau) [P\bar{V}_u^k](s_\tau, a_\tau) \rangle
&=
\langle \bm{\phi}, \Lambda_{k}^{-1} \sum_{\tau = 1}^{t_k - 1} \bm\varphi(s_\tau, a_\tau) \bm\varphi(s_\tau, a_\tau)^T \bm{w}_u^k{}^\ast \rangle \\
&=
\langle \bm{\phi}, \bm{w}_u^k{}^\ast \rangle - \lambda \langle \bm{\phi}, \Lambda_{k}^{-1} \bm{w}_u^k{}^\ast \rangle \\
&\leq
\langle \bm{\phi}, \bm{w}_u^k{}^\ast \rangle + \lambda \Vert \bm{\phi} \Vert_{\Lambda_{k}^{-1}} \Vert \bm{w}_u^k{}^\ast \Vert_{\Lambda_{k}^{-1}} \\
&\leq
\langle \bm{\phi}, \bm{w}_u^k{}^\ast \rangle + H\sqrt{\lambda d} \Vert \bm{\phi} \Vert_{\Lambda_{k}^{-1}}
\end{align*}
where the first inequality is by Cauchy-Schwartz and the second inequality is by Lemma~\ref{lemma:pv}.
Under the event $\mathcal{E}$ defined in Lemma~\ref{lemma:concentration-inequality}, the second term $(b)$ can be bounded by
\begin{align*}
\langle \bm{\phi}, \Lambda_{k}^{-1} \sum_{\tau = 1}^{t_k - 1} \bm\varphi(s_\tau, a_\tau) &(\bar{V}_u^k(s_{\tau + 1}) - [P\bar{V}_u^k](s_\tau, a_\tau)) \\
&\leq
\Vert \bm{\phi} \Vert_{\Lambda_{k}^{-1}}
\bigg\Vert \sum_{\tau = 1}^{t_k - 1} \bm\varphi(s_\tau, a_\tau) (V_u^k (s_{\tau + 1}) - [PV_u^k](s_\tau, a_\tau)) \bigg\Vert_{\Lambda_{k}^{-1}} \\
&\leq C \cdot H d \sqrt{\log((c_\beta + 1) dT / \delta)} \cdot \Vert \bm{\phi} \Vert_{\Lambda_{k}^{-1}}.
\end{align*}
Combining the two bounds and rearranging, we get
$$
\langle \bm{\phi}, \bm{w}_u^k - \bm{w}_u^k{}^\ast \rangle \leq C \cdot H d \sqrt{(\log(c_\beta + 1) d T / \delta)} \cdot \Vert \bm{\phi} \Vert_{\Lambda_{k}^{-1}}
$$
for some absolute constant $C$ independent of $c_\beta$.
Lower bound of $\langle \bm{\phi}, \bm{w}_u^k - \bm{w}_u^k{}^\ast \rangle$ can be shown similarly, establishing
$$
\vert \langle \bm{\phi}, \bm{w}_u^k - \bm{w}_u^k{}^\ast \rangle \vert \leq C \cdot H d \sqrt{\log((c_\beta + 1) d T / \delta)}\cdot \Vert \bm{\phi} \Vert_{\Lambda_{k}^{-1}}.
$$
It remains to show that there exists a choice of absolute constant $c_\beta$ such that
$$
C \sqrt{\log (c_\beta + 1) + \log(d T / \delta)} \leq c_\beta \sqrt{ \log(dT / \delta) }.
$$
Noting that $\log(d T / \delta) \geq \log 2$, this can be done by choosing an absolute constant $c_\beta$ that satisfies $C \sqrt{\log 2 + \log(c_\beta + 1)} \leq c_\beta \sqrt{\log 2}$.
\end{proof}

\subsection{Optimism} \label{appendix:optimism-linear}

\begin{proof}[Proof of Lemma~\ref{lemma:optimism-linear}]
We prove under the event $\mathcal{E}$ defined in Lemma~\ref{lemma:concentration-inequality}.
Fix any episode index $k > 1$.
We prove by induction on $u = T + 1, T, \dots, 1$.
The base case $u = T + 1$ is trivial since $V_{T + 1}^k(s) = \frac{1}{1 - \gamma} \geq V^\ast(s)$ for all $s \in \mathcal{S}$ and $Q_{T + 1}^k(s, a) = \frac{1}{1 - \gamma} \geq Q^\ast(s, a)$ for all $(s, a) \in \mathcal{S} \times \mathcal{A}$.

Now, suppose the optimism results $V_{u + 1}^k(s) \geq V^\ast(s)$ and $Q_{u + 1}^k(s, a) \geq Q^\ast(s, a)$ for all $(s, a) \in \mathcal{S} \times \mathcal{A}$ hold for some $u \in [T]$.
For convenience, we use the notation $\bar{V}_u^k(s) = V_u^k(s) - \min_{s'} V_u^k(s')$.
Using the concentration bounds of regression coefficients $\bm{w}_u^k$ provided in Lemma~\ref{lemma:concentration-regression}, which holds under the event $\mathcal{E}$, we can lower bound $Q_u^k(s, a)$ as follows.
\begin{align*}
Q_u^k(s, a)
&= \left(r(s, a) + \gamma ( \langle \bm\varphi(s, a), \bm{w}_{u + 1}^k \rangle + \min_{s'} V_{u + 1}^k (s') + \beta \Vert \bm\varphi(s, a) \Vert_{\Lambda_{k}^{-1}} \right) \wedge \frac{1}{1 - \gamma} \\
&\geq
\left( r(s, a) + \gamma (\langle \bm\varphi(s, a), \bm{w}_{u + 1}^k{}^\ast \rangle + \min_{s'} V_{u + 1}^k(s') ) \right) \wedge \frac{1}{1 - \gamma} \\
&= (r(s, a) + \gamma PV_{u + 1}^k(s, a)) \wedge \frac{1}{1 - \gamma} \\
&\geq (r(s, a) + \gamma PV^\ast(s, a)) \wedge \frac{1}{1 - \gamma} \\
&= Q^\ast(s, a)
\end{align*}
where $\bm{w}_{u + 1}^k{}^\ast$ is a parameter that satisfies $\langle \bm\varphi(s, a), \bm{w}_{u + 1}^k{}^\ast \rangle = [P \bar{V}_{u + 1}^k] (s, a)$.
The second inequality is by the induction hypothesis $V_{u + 1}^k \geq V^\ast$ and the last equality is by the Bellman optimality equation for the discounted setting and the fact that $Q^\ast \leq \frac{1}{1 - \gamma}$.

We established $Q_u^k(s, a) \geq Q^\ast(s, a)$ for all $(s, a) \in \mathcal{S} \times \mathcal{A}$.
It remains to show that $V_u^k(s) \geq V^\ast(s)$ for all $s \in \mathcal{S}$.
Recall that the algorithm defines $\widetilde{V}_u^k(\cdot) = \max_a Q_u^k(\cdot, a)$.
Hence, for all $s \in \mathcal{S}$, we have
\begin{align*}
\widetilde{V}_u^k(s) - V^\ast(s)
&= \max_a Q_u^k(s, a) - V^\ast(s) \\
&\geq Q_u^k(s, a_s^\ast) - Q^\ast(s, a_s^\ast) \\
&\geq 0
\end{align*}
where we use the notation $a_s^\ast = \argmax_a Q^\ast(s, a)$ so that $V^\ast(s) = Q^\ast(s, a_s^\ast)$, establishing $\widetilde{V}_u^k(s) \geq V^\ast(s)$ for all $s \in \mathcal{S}$.
Hence, for all $s \in \mathcal{S}$, we have
\begin{align*}
V_u^k(s) &= \widetilde{V}_u^k(s) \wedge (\min_{s'} \widetilde{V}_u^k(s') + 2 \cdot \text{sp}(v^\ast)) \\
&\geq V^\ast(s) \wedge (\min_{s'} V^\ast(s') + 2 \cdot \text{sp}(v^\ast)) \\
&= V^\ast(s)
\end{align*}
where the last equality is due to $\text{sp}(V^\ast) \leq 2\cdot \text{sp}(v^\ast)$ by Lemma~\ref{lemma:discounted-approximation}.
By induction, the proof for the optimism results $V_u^k(s) \geq V^\ast(s)$ and $Q_u^k(s, a) \geq Q^\ast(s, a)$ for $u = T + 1, T, \dots, 1$ is complete.
\end{proof}

\subsection{Access to $\min_{s'} V^\ast(s')$} \label{section:access-to-min}

In this section, we demonstrate that using $\min_{s'} V^\ast(s')$ for clipping instead of $\min_{s'} V_u^k(s')$ at the clipping step in Algorithm~\ref{alg:lscvi-ucb} achieves the same regret bound.

Since the only part of the proof affected by the modified clipping is the optimism proof, we provide the proof for the optimism lemma only.

\begin{lemma}[Optimism]
Under the linear MDP setting, consider running a modified version of Algorithm~\ref{alg:lscvi-ucb} that uses $\min_{s'} V^\ast(s')$ for clipping instead of $\min_{s'} V_u^k(s')$ at the $u$th value iteration step in episode $k$. Using the same input $H = 2 \cdot \text{sp}(v^\ast)$ for the modified algorithm guarantees with probability at least $1 - \delta$ that
for all $u = 1, \dots, T$ and for all $(s, a) \in \mathcal{S} \times \mathcal{A}$,
$$
V_u^k(s) \geq V^\ast(s), \quad Q_u^k(s, a) \geq Q^\ast(s, a).
$$
\end{lemma}

\begin{proof}
We prove under the event $\mathcal{E}$ defined in Lemma~\ref{lemma:concentration-inequality}.
Fix any episode index $k > 1$.
We prove by induction on $u = T + 1, T, \dots, 1$.
The base case $u = T + 1$ is trivial since $V_{T + 1}^k(s) = \frac{1}{1 - \gamma} \geq V^\ast(s)$ for all $s \in \mathcal{S}$ and $Q_{T + 1}^k(s, a) = \frac{1}{1 - \gamma} \geq Q^\ast(s, a)$ for all $(s, a) \in \mathcal{S} \times \mathcal{A}$.

Now, suppose the optimism results $V_{u + 1}^k(s) \geq V^\ast(s)$ and $Q_{u + 1}^k(s, a) \geq Q^\ast(s, a)$ for all $(s, a) \in \mathcal{S} \times \mathcal{A}$ hold for some $u \in [T]$.
For convenience, we use the notation $\bar{V}_u^k(s) = V_u^k(s) - \min_{s'} V^\ast(s')$.
Using the concentration bounds of regression coefficients $\bm{w}_{u + 1}^k$ provided in Lemma~\ref{lemma:concentration-regression}, which holds under the event $\mathcal{E}$, we can lower bound $Q_u^k(s, a)$ as follows.
\begin{align*}
Q_u^k(s, a)
&= \left(r(s, a) + \gamma ( \langle \bm\varphi(s, a), \bm{w}_{u + 1}^k \rangle + \min_{s'} V^\ast (s') + \beta \Vert \bm\varphi(s, a) \Vert_{\Lambda_{k}^{-1}} \right) \wedge \frac{1}{1 - \gamma} \\
&\geq
\left( r(s, a) + \gamma (\langle \bm\varphi(s, a), \bm{w}_{u + 1}^k{}^\ast \rangle + \min_{s'} V^\ast(s') ) \right) \wedge \frac{1}{1 - \gamma} \\
&= (r(s, a) + \gamma [PV_{u + 1}^k](s, a)) \wedge \frac{1}{1 - \gamma} \\
&\geq (r(s, a) + \gamma [PV^\ast](s, a)) \wedge \frac{1}{1 - \gamma} \\
&= Q^\ast(s, a)
\end{align*}
where $\bm{w}_{u + 1}^k{}^\ast$ is a parameter that satisfies $\langle \bm\varphi(s, a), \bm{w}_{u + 1}^k{}^\ast \rangle = [P \bar{V}_{u + 1}^k] (s, a)$.
The second inequality is by the induction hypothesis $V_{u + 1}^k \geq V^\ast$ and the last equality is by the Bellman optimality equation for the discounted setting and the fact that $Q^\ast \leq \frac{1}{1 - \gamma}$.

We established $Q_u^k(s, a) \geq Q^\ast(s, a)$ for all $(s, a) \in \mathcal{S} \times \mathcal{A}$.
It remains to show that $V_u^k(s) \geq V^\ast(s)$ for all $s \in \mathcal{S}$.
Recall that the algorithm defines $\widetilde{V}_u^k(\cdot) = \max_a \widetilde{Q}_u^k(\cdot, a)$.
Hence, for all $s \in \mathcal{S}$, we have
\begin{align*}
\widetilde{V}_u^k(s) - V^\ast(s)
&= \max_a Q_u^k(s, a) - V^\ast(s) \\
&\geq Q_u^k(s, a_s^\ast) - Q^\ast(s, a_s^\ast) \\
&\geq 0
\end{align*}
where we use the notation $a_s^\ast = \argmax_a Q^\ast(s, a)$ so that $V^\ast(s) = Q^\ast(s, a_s^\ast)$, establishing $\widetilde{V}_u^k(s) \geq V^\ast(s)$ for all $s \in \mathcal{S}$.
Hence, for all $s \in \mathcal{S}$, we have
\begin{align*}
V_u^k(s) &= \widetilde{V}_u^k(s) \wedge (\min_{s'} V^\ast(s') + 2 \cdot \text{sp}(v^\ast)) \\
&\geq V^\ast(s) \wedge (\min_{s'} V^\ast(s') + 2 \cdot \text{sp}(v^\ast)) \\
&= V^\ast(s)
\end{align*}
where the last equality is due to $\text{sp}(V^\ast) \leq 2 \cdot \text{sp}(v^\ast)$ by Lemma~\ref{lemma:discounted-approximation}.
By induction, the proof for the optimism results $V_u^k(s) \geq V^\ast(s)$ and $Q_u^k(s, a) \geq Q^\ast(s, a)$ for $u = T + 1, T, \dots, 1$ is complete.
\end{proof}

\subsection{Difficulty of Bounding Regret with Computationally Efficient Clipping} \label{section:difficulty}

For computational efficiency, consider using $\min_{s' \in \{s_1, \dots, s_{t_k} \}} \widetilde{V}_u^k(s')$ for clipping instead of $\min_{s' \in \mathcal{S}} \widetilde{V}_u^k(s')$ when running value iteration in the beginning of episode $k$ for generating value functions for episode $k$.
Note that in the beginning of episode $k$, we only have only seen states $s_1, \dots, s_{t_k}$.
A natural clipping operation for enforcing the span to be bounded by $H$ is
$$
V_u^k(\cdot) = (\widetilde{V}_u^k(\cdot) \wedge (\min_{s' \in \{s_1, \dots, s_{t_k} \}} \widetilde{V}_u^k(s') + H)) \vee \min_{s' \in \{s_1, \dots, s_{t_k} \}} \widetilde{V}_u^k(s'),
$$
which is equivalent to clipping the value $\widetilde{V}_u^k(\cdot)$ to the interval $[\min_{s' \in \{ s_1, \dots, s_{t_k} \}} \widetilde{V}_u^k(s'), \min_{s' \in \{ s_1, \dots, s_{t_k} \}} \widetilde{V}_u^k(s') + H]$.
One difficulty in the analysis that arises from this change in clipping is when bounding the term $\sum_{t = t_k}^{t_{k + 1} - 1} (V^k_{t + 1}(s_{t + 1}) - Q^k_t(s_t, a_t))$.
To illustrate the difficulty, recall one of the steps in the proof presented in this paper:
\begin{align*}
V^k_{t + 1}(s_{t + 1})
&\leq \widetilde{V}^k_{t + 1}(s_{t + 1}) \\
&= \max_a Q^k_{t + 1}(s_{t + 1}, a) \\
&= Q^k_{t + 1}(s_{t + 1}, a_{t + 1}).
\end{align*}
The inequality $V_{t + 1}^k(s_{t + 1}) \leq \widetilde{V}_{t + 1}^k(s_{t + 1})$ may no longer hold when using the new computationally efficient version of the minimum.
Instead, we get a bound that looks like
$$
\begin{aligned}
V_{t + 1}^k(s_{t + 1})
&= (\widetilde{V}_{t + 1}^k(s_{t + 1}) \wedge (\min_{s' \in \{s_1, \dots, s_{t_k} \}} \widetilde{V}_{t + 1}^k(s') + H)) \vee \min_{s' \in \{s_1, \dots, s_{t_k} \}} \widetilde{V}_{t + 1}^k(s') \\
&\leq \widetilde{V}_{t + 1}^k(s_{t + 1}) + (\min_{s' \in \{ s_1, \dots, s_{t_k} \}} \widetilde{V}_{t + 1}^k(s') - \widetilde{V}_{t + 1}^k(s_{t + 1}))_+
\end{aligned}
$$
where $(\cdot)_+ = \max\{ 0, \cdot \}$.
It is unclear how to bound the sum of $(\min_{s' \in \{ s_1, \dots, s_{t_k} \}} \widetilde{V}_{t + 1}^k(s') - \widetilde{V}_{t + 1}^k(s_{t + 1}))_+$ over $t = 1, \dots, T$.
We conjecture that additional algorithmic technique is required to proceed with the bound.

\section{Computational Complexity} \label{section:computation}

\subsection{$\gamma$-LSCVI-UCB (Algorithm~\ref{alg:lscvi-ucb})}

Our algorithm $\gamma$-LSCVI-UCB runs in episodes and the number of episodes is bounded by $\mathcal{O}(d \log T)$.
In each episode, value iteration is run for at most $T$ iterations.
In each iteration $u$ in episode $k$, one evaluation of $\min_{s'} \widetilde{V}_u^k(s')$, $t_k$ evaluations of $V_u^k(\cdot)$ of $V_u^k(\cdot)$ and a multiplication of $d \times d$ matrix ($\Lambda_k^{-1}$) and a $d$-dimensional vector is required.

One evaluation of $\min_{s'} \widetilde{V}_u^k(s')$ involves $S$ evaluations of $\widetilde{V}_u^k(\cdot)$.
One evaluation of $\widetilde{V}_u^k(\cdot)$ involves $A$ evaluations of $Q_u^k(\cdot, \cdot)$.
One evaluation of $Q_u^k(\cdot, \cdot)$ requires $\mathcal{O}(d^2)$ operations.
In total, one evaluation of $\min_{s'} \widetilde{V}_u^k(s')$ requires $\mathcal{O}(d^2 SA)$ operations.

Now, computing $\bm{w}_u^k$ requires evaluating $V_{u + 1}^k(\cdot)$ for at most $T$ states, which requires $\mathcal{O}(d^2 AT)$ operations; adding at most $T$ $d$-dimensional vectors, which requires $Td$ operations; and multiplying by $d \times d$ matrix, which requires $d^2$ operations.

In total, computing $\bm{w}_u^k$ requires $\mathcal{O}(d^2 A(S + T))$ operations.
Hence, running at most $T$ value iterations in each episode requires $\mathcal{O}(d^2 A(S+T)T)$ operations, and since there are at most $\mathcal{O}(d \log T)$ episodes, total operations for the algorithm is $\widetilde{\mathcal{O}}(d^3 A(S + T)T)$, which is polynomial in $d, S, A, T$.

\subsection{FOPO \parencite{wei2021learning}}

In this section, we provide time complexity analysis of the FOPO algorithm \parencite{wei2021learning}.
The algorithm is shown in Algorithm~\ref{alg:fopo}.

\begin{algorithm*}[t]
\KwInput{$\delta \in (0, 1)$, $\lambda = 1$, $\beta = 20(2 + \text{sp}(v^\ast) ) d \sqrt{\log(T / \delta)}$}
\KwInit{$\Lambda_1 = \lambda I$}
Receive initial state $s_1$. \\
\For{time step $t = 1, \dots, T$}{
Solve the following optimization problem to get $w_t$:
\begin{align*}
    \max_{w_t, b_t \in \mathbb{R}^d, J_t \in \mathbb{R}} ~~~&J_t \\
    \text{subject to}~~~&
    w_t = \Lambda_t^{-1} \sum_{\tau = 1}^{t - 1} ( \bm\varphi(s_\tau, a_\tau) (r(s_\tau, a_\tau) - J_t + \max_a \langle \bm\varphi(s_{\tau + 1}, a), w_t \rangle) + b_t) \\
    &\Vert b_t \Vert_{\Lambda_t} \leq \beta \\
    &\Vert w_t \Vert \leq (2 + \text{sp}(v^\ast)) \sqrt{d}
\end{align*}
}
\caption{FOPO}
\label{alg:fopo}
\end{algorithm*}

The bottleneck of the algorithm is solving the optimization problem.
The algorithm needs to solve the optimization problem $\mathcal{O}(d \log T)$ times since the number of episodes is $\mathcal{O}(d \log T)$.
Since there is no efficient way of solving the fixed point optimization problem to the best of our knowledge, we provide an analysis of the time complexity of a brute force approach for approximately solving the problem.
The brute force approach does a grid search on the optimization variables $w_t$, $b_t$, and $J_t$.

Consider the following grids:
\begin{align*}
G_w(\Delta) &= \{ \Delta \cdot (k_1, \dots, k_d) : \pm k_1, \dots, \pm k_d \in [\lfloor (2 + \text{sp}(v^\ast)) / \Delta \rfloor] \} \\
G_b(\Delta) &= \{ \Delta \cdot (k_1, \dots, k_d) : \pm k_1, \dots, \pm k_d \in [\lfloor T \beta / (\sqrt{d} \Delta) \rfloor] \} \\
G_J(\Delta) &= \{ \Delta k : k \in [\lfloor 1 / \Delta \rfloor \}.
\end{align*}
The grids are designed such that the constraints $\Vert w \Vert_2 \leq (2 + \text{sp}(v^\ast)) \sqrt{d}$, $\Vert b \Vert_{\Lambda_t} \leq \beta$ are satisfied for all $w \in G_w(\Delta)$ and $b \in G_b(\Delta)$ and $J \in [0, 1]$ for all $J \in G_J(\Delta)$.
Also, $G_w(\Delta)$, $G_b(\Delta)$ and $G_J(\Delta)$ are $\Delta$-covering with respect to $\Vert \cdot \Vert_\infty$ of $\mathcal{B}_d((2 + \text{sp}(v^\ast)) \sqrt{d})$, $\mathcal{B}_d(T \beta)$ and $[0, 1]$ respectively, where $\mathcal{B}_d(r)$ is a $d$-dimensional ball of radius $r$.

Denote by $\Delta_t(w, b, J)$ the difference between the left hand side and the right hand side of the fixed point equation at time step $t$, i.e.,
$$
\Delta_t(w, b, J) = w - \Lambda_t^{-1} \sum_{\tau = 1}^{t - 1} ( \bm\varphi(s_\tau, a_\tau) (r(s_\tau, a_\tau) - J + \max_a \langle \bm\varphi(s_{\tau + 1}, a), w \rangle) + b).
$$
Let $w_t^\ast, b_t^\ast, J_t^\ast$ be the solution to the fixed point problem, which may not lie in the grids.
Then, $\Delta_t(w_t^\ast, b_t^\ast, J_t^\ast) = 0$.
Let $\widetilde{w}_t^\ast \in G_w(\Delta)$, $\widetilde{b}_t^\ast \in G_b(\Delta)$ and $\widetilde{J}^\ast \in G_J(\Delta)$ be the grid points closest to $w_t^\ast, b_t^\ast, J_t^\ast$, respectively.
Then,
\begin{align*}
&\Vert \Delta_t(\widetilde{w}_t^\ast, \widetilde{b}_t^\ast, \widetilde{J}^\ast) \Vert_2 \\
&= \Vert \Delta_t(\widetilde{w}_t^\ast, \widetilde{b}_t^\ast, \widetilde{J}^\ast) - \Delta_t(w_t^\ast, b_t^\ast, J_t^\ast) \Vert_2 \\
&= \Vert \widetilde{w}_t^\ast - w_t^\ast - \Lambda_t^{-1} \sum_{\tau = 1}^{t - 1} ( \bm\varphi(s_\tau, a_\tau)(J_t^\ast - \widetilde{J}_t^\ast + \max_a \langle \bm\varphi(s_{\tau + 1}, a), \widetilde{w}_t^\ast \rangle - \max_a \langle \bm\varphi(s_{\tau + 1}, a), w_t^\ast \rangle) + \widetilde{b}_t^\ast - b_t^\ast) \Vert_2 \\
&\leq
\Vert \widetilde{w}_t^\ast - w_t^\ast \Vert_2 + \sum_{\tau = 1}^{t - 1} \vert J_t^\ast - \widetilde{J}_t^\ast \vert \Vert \Lambda_t^{-1} \Vert_2 \Vert \bm\varphi(s_\tau, a_\tau) \Vert_2
+ \sum_{\tau = 1}^{t - 1} \max_a \vert \langle \bm\varphi(s_{\tau + 1}, a), \widetilde{w}_t^\ast - w_t^\ast \rangle \vert \Vert \Lambda_t^{-1} \Vert_2 \Vert \bm\varphi(s_\tau, a_\tau) \Vert_2 \\
&\hspace{30mm}+ \sum_{\tau = 1}^{t - 1} \Vert \widetilde{b}_t^\ast - b_t^\ast \Vert_2 \Vert \Lambda_t^{-1} \Vert_2 \\
&\leq \Delta \sqrt{d} + T \Delta + T \Delta \sqrt{d} + T \Delta \sqrt{d} \\
&\leq \mathcal{O}(T \sqrt{d} \Delta).
\end{align*}
Hence, the solution $\widetilde{w}_t, \widetilde{b}_t, \widetilde{J}_t$ obtained by the grid search satisfies
$$
\Vert \Delta_t(\widetilde{w}_t, \widetilde{b}_t, \widetilde{J}_t) \Vert_2 \leq \Vert \Delta_t (\widetilde{w}_t^\ast, \widetilde{b}_t^\ast, \widetilde{J}_t^\ast) \Vert_2 \leq \mathcal{O}(T \sqrt{d} \Delta).
$$
Inspecting the proof in Appendix C.1 in \textcite{wei2021learning}, it can be seen that the additional regret incurred by approximating the solution of the fixed point problem is
$$
\sum_{t = 1}^T \langle \bm\varphi(s_t, a_t), \Delta_t(\widetilde{w}_t, \widetilde{b}_t, \widetilde{J}_t) \rangle
\leq T^2 \sqrt{d} \Delta.
$$
Choosing the grid size $\Delta$ to be $\mathcal{O}(1 / \sqrt{T^3})$ guarantees the additional regret does not affect the order of the total regret.
Since the grid search requires $\mathcal{O}(((1 + \text{sp}(v^\ast)) / \Delta)^d \times (T \beta / (\sqrt{d} \Delta))^d \times (1 / \Delta)) = \widetilde{\mathcal{O}}((T(1 + \text{sp}(v^\ast))^2 \sqrt{d})^d (1 / \Delta)^{2d + 1})$ evaluations of the fixed point equation, the grid search method requires $\widetilde{\mathcal{O}}(T^{4d + 3/2} (1 + \text{sp}(v^\ast))^{2d} d^{d/2})$ evaluations, and each evaluation requires $\mathcal{O}(d^2 + T)$ operations.
In total, FOPO can be run using the brute force grid search method with time complexity $\widetilde{\mathcal{O}}(T^{4d + 5 / 2} (1 + \text{sp}(v^\ast))^{2d} d^{d / 2 + 3})$, which is exponential in $d$.

\subsection{LOOP \parencite{he2024sample}}

The LOOP algorithm \parencite{he2024sample} solves the following optimization problem every episode:

\begin{align*}
\max_{w_t \in \mathbb{B}_d(R), J_t \in [0, 1]}~~& J_t \\
\text{subject to}~~& \sum_{\tau = 1}^{t - 1} (\langle \bm\varphi(s_\tau, a_\tau), \bm{w}_t \rangle - r(s_\tau, a_\tau) - \max_a \langle \bm\varphi(s_{\tau + 1}, a), \bm{w}_t \rangle + J_t)^2 \\
&\hspace{15mm} \min_{w' \in \mathbb{B}_d(R), J' \in [0, 1]} \sum_{\tau = 1}^{t - 1} (\langle \bm\varphi(s_\tau, a_\tau), \bm{w}_t \rangle - r(s_\tau, a_\tau) - \max_a \langle \bm\varphi(s_{\tau + 1}, a), \bm{w}' \rangle + J')^2 \leq \beta
\end{align*}
where $R = \frac{1}{2} \text{sp}(v^\ast) \sqrt{d}$ and $\beta = \widetilde{\mathcal{O}}(\text{sp}(v^\ast) d)$.
To the best of our knowledge, there is no computationally efficient way of solving this problem.
Solving the problem by grid search involves looping over $\text{poly}(T^d)$ grid points for $\bm{w}_t$ and $J_t$. Also, checking the constraint for each grid point requires $\text{poly}(T, d, A)$ operations.
Hence, the total time complexity of the algorithm is $\text{poly}(T^d, d, A)$.

\section{Other Technical Lemmas}

\begin{lemma} [Lemma D.1 in \textcite{jin2020provably}] \label{lemma:fixed-bonus-sum}
Let $\Lambda_t = \sum_{i = 1}^t \bm\phi_i \bm\phi_i^T + \lambda I$ where $\bm\phi_i \in \mathbb{R}^d$ and $\lambda > 0$. Then,
$$
\sum_{i = 1}^t \bm\phi_i^T \Lambda_t^{-1} \bm\phi_i \leq d.
$$
\end{lemma}

\begin{lemma}[Lemma 11 in \textcite{abbasi2011improved}] \label{lemma:bonus-term-linear}
Let $\{ \bm\phi_t \}_{t \geq 1}$ be a bounded sequence in $\mathbb{R}^d$ with $\Vert \bm\phi_t \Vert_2 \leq 1$ for all $t \geq 1$.
Let $\Lambda_0 = I$ and $\Lambda_t = \sum_{i = 1}^t \bm\phi_i \bm\phi_i^T + I$ for $t \geq 1$. Then,
$$
\sum_{i = 1}^t \bm\phi_i^T \Lambda_{i - 1}^{-1} \bm\phi_i \leq 2 \log \det(\Lambda_t) \leq 2 d \log(1 + t).
$$
\end{lemma}

\begin{lemma}[Lemma 12 in \textcite{abbasi2011improved}] \label{lemma:doubling}
Suppose $A, B \in \mathbb{R}^{d \times d}$ are two positive definite matrices satisfying $A \succeq B$.
Then, for any $\bm{x} \in \mathbb{R}^d$, we have
$$
\Vert \bm{x} \Vert_A \leq \Vert \bm{x} \Vert_B \sqrt{\frac{\det(A)}{\det(B)}}.
$$
\end{lemma}

\begin{lemma}[Bound on number of episodes] \label{lemma:num-episode}
 The number of episodes $K$ in Algorithm \ref{alg:lscvi-ucb} is bounded by 

 $$
 K \le d \log_2\left(1+\frac{T}{\lambda d}\right).
 $$
\end{lemma}
\begin{proof}

Let $\{\Lambda_k \}_{k = 1}^K$ and $\{ \bar\Lambda_t \}_{t = 0}^T$ be as defined in Algorithm \ref{alg:lscvi-ucb}. Note that
$$
\text{tr}(\bar\Lambda_T)=\text{tr}(\lambda I_d)+\sum_{t=1}^T \text{tr}(\varphi(s_t,a_t)\varphi(s_t,a_t)^T)=\lambda d+\sum_{t=1}^T\Vert\varphi(s_t,a_t)\Vert_2^2\le\lambda d+T.
$$
By the AM–GM inequality, we have 
$$
\det(\bar\Lambda_T)\le \left(\frac{\text{tr}(\bar\Lambda_T)}{d}\right)^d\le\left(\frac{\lambda d+T}{d}\right)^d.
$$

Since we update $\Lambda_k$ only when $\det (\bar\Lambda_t)$ doubles, $\text{det}(\bar\Lambda_T)\ge \text{det}(\Lambda_K)\ge\text{det}(\Lambda_1)\cdot 2^K=\lambda^d\cdot 2^K$. Thus, we obtain 
$$
K\le d\log_2\left(1+\frac{T}{\lambda d}\right)
$$
as desired.

\end{proof}

\section{Additional Related Work} \label{section:additional-related-work}

\begin{table*}[t]
\caption{Comparison of algorithms for infinite-horizon average-reward RL in tabular setting}
\label{table:comparison-tabular}
\centering
\begin{tabular}{ccccc}
 \toprule
 Algorithm & Regret $\widetilde{\mathcal{O}}(\cdot)$ & Assumption & Computation \\
 \midrule
 UCRL2 \parencite{auer2008near} & $DS \sqrt{A T}$ & Bounded diameter & Efficient \\
 REGAL \parencite{bartlett2009regal} & $\text{sp}(v^\ast) \sqrt{SAT}$ & Weakly communicating & Inefficient \\
 PSRL \parencite{ouyang2017learning} & $\text{sp}(v^\ast) S\sqrt{AT}$ & Weakly communicating & Efficient \\
 OSP \parencite{ortner2020regret} & $\sqrt{t_{\text{mix}} SAT}$ & Ergodic & Inefficient \\
 SCAL \parencite{fruit2018efficient} & $\text{sp}(v^\ast) S \sqrt{AT}$ & Weakly communicating & Efficient \\
 UCRL2B \parencite{fruit2020improved} & $S\sqrt{DAT}$ & Bounded diameter & Efficient \\
 EBF \parencite{zhang2019regret} & $\sqrt{\text{sp}(v^\ast) SAT}$ & Weakly communicating & Inefficient \\
 \textbf{$\gamma$-UCB-CVI (Ours)} & $\text{sp}(v^\ast)S\sqrt{AT}$ & Bellman optimality equation & Efficient \\
 \midrule
 Optimistic Q-learning \parencite{wei2020model} & $\text{sp}(v^\ast) (SA)^{\frac{1}{3}} T^{\frac{2}{3}}$ & Weakly communicating & Efficient \\
  MDP-OOMD \parencite{wei2020model} & $\sqrt{t_{\text{mix}}^3 \eta AT}$ & Ergodic & Efficient \\
  UCB-AVG \parencite{zhang2023sharper} & $\text{sp}(v^\ast) S^5 A^2  \sqrt{T}$ & Weakly communicating & Efficient \\
  \midrule
  Lower bound \parencite{auer2008near} & $\Omega(\sqrt{DSAT})$ & \\
 \bottomrule
\end{tabular}
\end{table*}

\paragraph{Infinite-Horizon Average-Reward Setting with Tabular MDP}
We focus on works on infinite-horizon average-reward setting with tabular MDP that assume either the MDP is weakly communicating or the MDP has a bounded diameter.
For other works and comparisons, see Table~\ref{table:comparison-tabular}.
Seminal work by \textcite{auer2008near} on infinite-horizon average-reward setting in tabular MDPs laid the foundation for the problem.
Their model-based algorithm called UCRL2 constructs a confidence set on the transition model and run an extended value iteration that involves choosing the optimistic model in the confidence set each iteration.
They achieve a regret bound of $\mathcal{O}(DS \sqrt{AT})$ where $D$ is the diameter of the true MDP.
\textcite{bartlett2009regal} improve the regret bound of UCRL2 by restricting the confidence set of the model to only include models such that the span of the induced optimal value function is bounded. Their algorithm, called REGAL, achieves a regret bound that scales with the span of the optimal value function $\text{sp}(v^\ast)$ instead of the diameter of the MDP. However, REGAL is computationally inefficient.
\textcite{fruit2018efficient} propose a model-based algorithm called SCAL, which is a computationally efficient version of REGAL.
\textcite{zhang2019regret} propose a model-based algorithm called EBF that achieves the minimax optimal regret of $\mathcal{O}(\sqrt{\text{sp}(v^\ast) SAT})$ by maintaining a tighter model confidence set by making use of the estimate for the optimal bias function.
However, their algorithm is computationally inefficient.
There is another line of work on model-free algorithms for this setting.
\textcite{wei2020model} introduce a model-free Q-learning-based algorithm called Optimistic Q-learning.
Their algorithm is a reduction to the discounted setting.
Although model-free, their algorithm has a suboptimal regret of $\mathcal{O}(T^{2/3})$.
Recently, \textcite{zhang2023sharper} introduce a Q-learning-based algorithm called UCB-AVG that achieves regret bound of $\mathcal{O}(\sqrt{T})$.
Their algorithm, which is also a reduction to the discounted setting, is the first model-free to achieve the order optimal regret bound.
Their main idea is to use the optimal bias function estimate to increase statistical efficiency.
\textcite{agrawal2024optimistic} introduces a model-free Q-learning-based algorithm and provides a unified view of episodic setting and infinite-horizon average-reward setting. However, their algorithm requires additional assumption of the existence of a state with bounded hitting time.

\paragraph{Infinite-Horizon Average-Reward Setting with General Function Approximation}

\textcite{he2024sample} study infinite-horizon average reward with general function approximation.
They propose an algorithm called LOOP which is a modified version of the fitted Q-iteration with optimistic planning and lazy policy updates.
Although their algorithm when adapted to the linear MDP set up achieves $\mathcal{O}(\sqrt{\text{sp}(v^\ast)^3 d^3 T})$, which is comparable to our work, their algorithm is computationally inefficient.

\paragraph{Infinite-Horizon Average-Reward Setting with Linear MDPs}
There is another work by \textcite{ghosh2023achieving} on the infinite-horizon average-reward setting with linear MDPs.
They study a more general constrained MDP setting where the goal is to maximize average reward while minimizing the average cost.
They achieve $\widetilde{\mathcal{O}}(\text{sp}(v^\ast) \sqrt{d^3 T})$ regret, same as our work, but they make an additional assumption that the optimal policy is in a smooth softmax policy class.
Also, their algorithm requires solving an intractable optimization problem.

\paragraph{Reduction of Average-Reward to Finite-Horizon Episodic Setting}
There are works that reduce the average-reward setting to the finite-horizon episodic setting.
However, in general, this reduction can only give regret bound of $\mathcal{O}(T^{2/3})$.
\textcite{chen2022learning} study the constrained tabular MDP setting and propose an algorithm that uses the finite-horizon reduction. Their algorithm gives regret bound of $\mathcal{O}(T^{2/3})$.
\textcite{wei2021learning} study the linear MDP setting and propose a finite-horizon reduction that uses the LSVI-UCB \parencite{jin2020provably}.
Their reduction gives regret bound of $\mathcal{O}(T^{2/3})$.

\paragraph{Online RL in Infinite-Horizon Discounted Setting}
The literature on online RL in the infinite-horizon discounted setting is sparse because there is no natural notion of regret in this setting without additional assumption on the structure of the MDP.
The seminal paper by \textcite{liu2020regret} introduce a notion of regret in the discounted setting and propose a Q-learning-based algorithm for the tabular setting and provides a regret bound. 
\textcite{he2021nearly} propose a model-based algorithm that adapts UCBVI \parencite{azar2017minimax} to the discounted setting and achieve a nearly minimax optimal regret bound.
\textcite{ji2024regret} propose a model-free algorithm with nearly minimax optimal regret bound.

\end{document}